\documentclass{article}

\usepackage[utf8]{inputenc} % allow utf-8 input
\usepackage[T1]{fontenc}    % use 8-bit T1 fonts
\usepackage{hyperref}       % hyperlinks
\usepackage{url}            % simple URL typesetting
\usepackage{booktabs}       % professional-quality tables
\usepackage{amsfonts}       % blackboard math symbols
\usepackage{nicefrac}       % compact symbols for 1/2, etc.
\usepackage{microtype}      % microtypography
\usepackage{amsmath}
\usepackage{amsthm}
\usepackage{color}
\usepackage{amssymb}
\usepackage{graphics}
\usepackage{graphicx}
\usepackage{subfigure}
\usepackage{float}
\usepackage{stackengine}
\newtheorem{theorem}{Theorem}
\newtheorem{lemma}{Lemma}
\newtheorem{exmp}{Example}[section]

\usepackage{mathrsfs}
\usepackage{multirow}
\usepackage{capt-of}

\makeatletter
\newif\if@restonecol
\makeatother

\usepackage[linesnumbered,ruled,vlined]{algorithm2e}%[ruled,vlined]{
\usepackage{algpseudocode}
\usepackage{amsmath}
\usepackage{changes}
\definechangesauthor[name={Per cusse}, color=orange]{per}
\setremarkmarkup{(#2)}

\title{Deterministic Policy Gradients With General State Transitions}
\author{Qingpeng Cai, Ling Pan, Pingzhong Tang\\
Institute for Interdisciplinary Information Sciences\\
  Tsinghua University\\
  \texttt{\{cqp14,pl17\}@mails.tsinghua.edu.cn},
  \texttt{kenshinping@gmail.com} \\}
\begin{document}
\maketitle

\begin{abstract}

We study a reinforcement learning setting, where the state transition function is a convex combination of a stochastic continuous function and a deterministic function. Such a setting generalizes the widely-studied stochastic state transition setting, namely the setting of deterministic policy gradient (DPG).

We firstly give a simple example to illustrate that the deterministic policy gradient may be infinite under deterministic state transitions, and introduce a theoretical technique to prove the existence of the policy gradient in this generalized setting. Using this technique,
we prove that the deterministic policy gradient indeed exists for a certain set of discount factors, and further prove two conditions that guarantee the existence for all discount factors.
We then derive a closed form of the policy gradient whenever exists. Furthermore, to overcome the challenge of high sample complexity of DPG in this setting, we propose the {\em Generalized Deterministic Policy Gradient} (GDPG) algorithm. The main innovation of the algorithm is a new method of applying model-based techniques to the model-free algorithm, the deep deterministic policy gradient algorithm (DDPG). GDPG optimize the long-term rewards of the model-based augmented MDP subject to a constraint that the long-rewards of the MDP is less than the original one.

We finally conduct extensive experiments comparing GDPG with state-of-the-art methods and the direct model-based extension method of DDPG on several standard continuous control benchmarks.
Results demonstrate that GDPG substantially outperforms DDPG, the model-based extension of DDPG and other baselines in terms of both convergence and long-term rewards in most environments.
\end{abstract}

\section{Introduction}

Reinforcement learning has been one of the most successful computational tools for solving complex decision making problems \cite{sutton1998reinforcement}, with extensive applications in both discrete tasks such as general game playing \cite{mnih2013playing, mnih2015human} and continuous control tasks such as robotics \cite{kober2013reinforcement}. In contrast to the traditional value-based methods \cite{tesauro1995temporal, watkins1992q,mnih2013playing,mnih2015human} that are meant for solving problems with discrete and low-dimensional action space, policy gradient methods \cite{peters2011policy, sutton2000policy} aim to tackle these limitations, by optimizing a parameterized policy via estimating the gradient of the expected long-term reward, using gradient ascent.

\cite{silver2014deterministic} propose the deterministic policy gradient (DPG) algorithm that aims to find an optimal deterministic policy, which lowers the variance when estimating the policy gradient \cite{zhao2011analysis}, compared to stochastic policies  \cite{sutton2000policy}. It is shown that the algorithm can be applied to domains with continuous and high-dimensional action spaces. \cite{lillicrap2015continuous}  further  propose the deep deterministic policy gradient (DDPG) algorithm, by combining deep neural networks to improve convergence. It is recognized that DDPG has been successful in robotic control tasks such as locomotion \cite{song2017recurrent} and manipulation \cite{gu2017deep}.

Despite the effectiveness of DDPG in these tasks, it is limited for problems with stochastic continuous state transitions.
Here, the continuity means that the probability density of the next state is continuous in the action taken at the current state. In fact, many important control problems, such as MountainCar, Pendulum \cite{1606.01540}, and autonomous driving, include both stochastic and deterministic state transitions. For example, in most autonomous driving tasks, state transitions are deterministic under normal driving conditions, yet are still stochastic due to sudden disturbances. As a result,  DDPG, which assumes stochastic state transitions, does not generalize well in practice.

Tasks with deterministic state transitions pose serious technical challenges due to the discontinuity of the transition function, where the gradient of the transition probability density function over actions does not always exist. \cite{werbos1990menu, fairbank2008reinforcement, heess2015learning} consider the gradient of the value function over states and the deterministic policy gradient in the setting of deterministic state transitions, but the existence of the value function's gradient over states and the deterministic policy gradient is not studied.
Lacking of theoretical guarantees for the existence of the gradient limits the applicability of deterministic policy gradient algorithms. As a result, an important question for policy gradient based methods is,\\
{\em Does the gradient exist in settings with deterministic state transitions? If yes, can one solve the problem efficiently by its gradient? }

In this paper, we study a generalized setting, where the state transition is a convex combination of a stochastic continuous transition function and a deterministic discontinuous transition function. As a result, it includes both the stochastic case and the deterministic case as special cases. Our setting is arguably closer to the mixed control problems mentioned above than those stochastic settings.  We first give a simple example to illustrate that the deterministic policy gradient may be infinite under deterministic state transitions.  Then we introduce a new theoretical technique to prove the existence of the policy gradient in this generalized setting. Using this technique, we prove that the deterministic policy gradient indeed exists for a certain set of discount factors. We further present two conditions that guarantee the existence for all discount factors. We then derive a closed form of the policy gradient.

However, the estimation of the deterministic policy gradient is much more challenging due to the sample complexity of model-free algorithms  \cite{schulman2015trust} and complex state transitions. As for the state transition, the difficulty of the computation of the gradient mainly comes from the dependency of the policy gradient and the gradient of the value function over the state. Such computation may involve infinite times of sampling the whole state space. Thus applying DDPG directly in a general setting even with low-dimensional state space may incur high sample complexity.

To overcome these challenges, we approximate the original Markov decision process (MDP) by a model-based augmented MDP with the same reward function and the transition function being the expectation of original MDP. By the form of the deterministic policy gradient with deterministic state transitions, we get that the model-based augmented MDP has a simple structure, which allows for more efficient computations and faster convergence than model-free methods \cite{li2004iterative, levine2013guided, watter2015embed}. Unfortunately, applying this mode-based technique directly does not help to solve environments with large continuous state space as it is hard to represent the transition dynamics \cite{wahlstrom2015pixels}. 
This leads to an essential question:\\
{\em How to apply the model-based technique to deterministic policy gradient algorithms effectively? }

We then consider a program that maximizes the long-term rewards of the augmented MDP with the constraint that its long-term rewards is less than that of the original MDP. The intuition is that we choose a objective with less sample complexity to optimize, and it serves as a lower bound of the original objective. Note that the improvement of the new objective, guarantees the improvement of the original objective.  As the constrainted problem is hard to optimize, we choose to optimize the Lagrangian dual function of the program, which can be interpreted as a weighted objective between the long-term reward of the original MDP and the augmented MDP. Based on this dual function, we propose the {\em Generalized Deterministic Policy Gradient (GDPG) algorithm}. The algorithm updates the policy by stochastic gradient ascent with the gradient of the weighted objective over the parameter of the policy, and the weight maintains a trade-off between fast convergence and performance.

To sum up, the main contribution of the paper is as follows:

\begin{itemize}
\item First of all, we provide a theoretical guarantee for the existence of the gradient in settings with deterministic state transitions.
\item Secondly, we propose a novel policy gradient algorithm, called Generalized Deterministic Policy Gradient (GDPG), which combines the model-free and model-based methods.
GDPG reduces sample complexity, enables faster convergence and performance improvement.
\item Finally, we conduct extensive experiments on standard benchmarks comparing with state-of-the-art stochastic policy gradient methods including TRPO \cite{schulman2015trust}, ACKTR \cite{wu2017scalable} and the direct model-based extension of DDPG, called MDPG.
Results confirm that GDPG significantly outperforms other algorithms in terms of both convergence and performance.
\end{itemize}

\section{Preliminaries}

A Markov decision process (MDP) is a tuple $(S, A, p, r, \gamma, p_1)$, where $\mathcal{S}$ and $\mathcal{A}$ denote the set of states and actions respectively.
Let $p(s_{t+1} | s_t, a_t)$ represent the conditional density from state $s_t$ to state $s_{t+1}$ under action $a_t$, which satisfies the Markov property, i.e., $p(s_{t+1} | s_0, a_0, ..., s_t, a_t) = p(s_{t+1}| s_t, a_t).$
The density of the initial state distribution is denoted by $p_0(s)$.
At each time step $t$, the agent interacts with the environment with a deterministic policy $\mu_{\theta}$, which is parameterized by $\theta$.
We use $r(s_t, a_t)$ to represent the corresponding immediate reward, contributing to the discounted overall rewards from state $s_0$ following $\mu_{\theta}$, denoted by $J(\mu_{\theta}) = \mathbb{E}[\sum_{k=0}^{\infty}{\gamma}^{k}r(a_k, s_k)|\mu_{\theta},s_0]$.
Here, $\gamma \in [0, 1]$ is the discount factor.
The Q-function of state $s_t$ and action $a_t$ under policy $\mu_{\theta}$ is denoted by $Q^{\mu_{\theta}}(s_t,a_t) =  \mathbb{E}[\sum_{k=t}^{\infty}{\gamma}^{k-t}r(a_k, s_k)|\mu_{\theta},s_t,a_t]$.
The corresponding value function of state $s_t$ under policy $\mu_{\theta}$ is denoted by $V^{\mu_{\theta}}(s_t)=Q^{\mu_{\theta}}(s_t,\mu_{\theta}(s_t))$.
We denote the density at state $s^{'}$ after $t$ time steps from state $s$ by $p(s,s^{'},t,\mu_{\theta})$ following the policy $\mu_{\theta}$.
We denote the (improper) discounted state distribution by $\rho^{\mu_{\theta}}(s^{'})=\int_{\mathcal{S}}^{}\sum_{t=1}^{\infty}{\gamma}^{t-1}p_0(s)p(s,s^{'},t,\mu_{\theta})ds$.
The agent aims to find an optimal policy that maximizes $J(\mu_{\theta})$.

\subsection{Why is the DPG theorem not applicable for deterministic state transitions?}
An important property of the DPG algorithms is the Deterministic Policy Gradient Theorem \cite{silver2014deterministic},
$
\bigtriangledown_{\theta}J(\mu_{\theta})=\int_{\mathcal{S}}\rho^{\mu_{\theta}}(s)(\bigtriangledown_{\theta}\mu_{\theta}(s)\bigtriangledown_{a}Q^{\mu_{\theta}}(s,a)|_{a=\mu_{\theta}(s)})ds,$ which proves the existence of the deterministic policy gradient. The DPG theorem holds under the regular condition presented by \cite{silver2014deterministic}, i.e., $p(s^{'}|s,a)$ is continuous in $a$. The arguments in the proof of the DPG theorem do not work without this condition\footnote{Readers can refer to http://proceedings.mlr.press/v32/silver14-supp.pdf}. 

Now we give a simple example to show the policy gradient is infinite for some discount factors.

\begin{exmp}
\label{example}
Given a MDP with two dimensional state spaces and action spaces, whose transition and reward functions are defined by

$T(s,a)={(2s_1+2s_2+a_1,2s_1+2s_2+a_2)}^{T}$, $r(s,a)=-s^{T}a$.
Consider a deterministic policy $\mu_{\theta}(s)=\theta$, then $\bigtriangledown_{s}T(s,\mu_{\theta}(s))=\begin{bmatrix}
 2&2 \\ 
 2& 2
\end{bmatrix}$,
and $\bigtriangledown_{s}V^{\mu_{\theta}}(s)=-(I+\sum_{n=1}^{\infty}{\gamma}^{n}
\begin{bmatrix}
 2^{2n-1}&2^{2n-1} \\ 
 2^{2n-1}& 2^{2n-1}
\end{bmatrix})\theta. $
Then $\bigtriangledown_{s}V^{\mu_{\theta}}(s)$ converges if and only if $\gamma<1/4$. 
\end{exmp}

One must need a new technique to determine the existence of the gradient of $J(\mu_{\theta})$ over $\theta$ in irregular cases.

\section{Deterministic State Transitions}
\label{Deter}
In this section we study a simple setting where the state transition is a deterministic function. As discussed before, the DPG theorem does not apply to this setting. To analyze the gradient of a deterministic policy, we let $T(s,a)$ denote the next state given the current state $s$ and the action $a$. Without loss of generality, we assume that $T(s,a),\bigtriangledown_aT(s,a), \bigtriangledown_sT(s,a), r(s,a), \bigtriangledown_sr(s,a), \bigtriangledown_ar(s,a)$ are all continuous in $s$ and $a$ and bounded. By definition, $
\bigtriangledown_{\theta} V^{\mu_{\theta}}(s) = \bigtriangledown_{\theta} (r(s,\mu_{\theta}(s))+\gamma  V^{\mu_{\theta}}(s^{'})|_{s^{'}=T(s,\mu_{\theta}(s))}).$
Thus the key of  the existence of the gradient of $V^{\mu_{\theta}}(s)$ over $\theta$ is the existence of $\bigtriangledown_{s}V^{\mu_{\theta}}(s)$. Now we give a sufficient condition of the existence of $\bigtriangledown_{s}V^{\mu_{\theta}}(s)$.

\begin{lemma}
\label{lemma1}
For any policy $\mu_{\theta}$, let $n$ denote the dimension of the state, and $c$ be the maximum of the max norm of all Jacobain matrices, $\max_{s}||\bigtriangledown_{s}T(s,\mu_{\theta}(s))||_{max}$, for any discount factor $\gamma$ in $[0,\frac{1}{nc})$, 
$\bigtriangledown_{s}V^{\mu_{\theta}}(s)$ exists.
\end{lemma}

\begin{proof}
By definition,
$
V^{\mu_{\theta}}(s)=Q^{\mu_{\theta}}(s,\mu_{\theta}(s))=r(s,\mu_{\theta}(s))+\gamma V^{\mu_{\theta}}(s^{'})|_{s^{'}=T(s,\mu_{\theta}(s))}). 
$
Then 
\begin{equation}
\label{v_s}
\begin{split}
\bigtriangledown_{s}V^{\mu_{\theta}}(s)&=\bigtriangledown_{s}r(s,\mu_{\theta}(s))+\gamma \bigtriangledown_{s}T(s,\mu_{\theta}(s)) \bigtriangledown_{s^{'}}V^{\mu_{\theta}}(s^{'})|_{s^{'}=T(s,\mu_{\theta}(s))}.
%&=\bigtriangledown_{s}r(s,a)|_{a=\mu_{\theta}(s)}+\bigtriangledown_{a}r(s,a)|_{a=\mu_{\theta}(s)}\bigtriangledown_{s}\mu_{\theta}(s)+\gamma c  \bigtriangledown_{s^{'}}V^{\mu_{\theta}}(s^{'})|_{s^{'}=T(s,\mu_{\theta}(s))}.
\end{split}
\end{equation}
%Then if we choose a value of $\gamma$ such that $\gamma c<1$, it is easy to see that
By unrolling (\ref{v_s}) with infinite steps, we get

$$\bigtriangledown_{s}V^{\mu_{\theta}}(s)=\sum_{t=0}^{\infty}\int_{\mathcal{S}}{\gamma}^{t}g(s,t,\mu_{\theta})I(s,s^{'},t,\mu_{\theta})\bigtriangledown_{s^{'}}r(s^{'},\mu_{\theta}(s^{'}))ds^{'},$$
where $I(s,s^{'},t,\mu_{\theta})$ is an indicator function that indicates whether $s^{'}$ is obtained after $t$ steps from the state $s$ following the policy $\mu_{\theta}$.
Here, $g(s,t,\mu_{\theta})=\prod_{i=0}^{t-1}\bigtriangledown_{s_i}T(s_i,\mu_{\theta}(s_i)), $ where $s_0=s$ and $s_i$ is the state after $i$ steps following policy $\mu_{\theta}$.
The state transitions and policies are both deterministic. 
We now prove that for any $\mu_{\theta},s,s^{'}$ and $\gamma\in [0,\frac{1}{nc})$, $A(s)=\sum_{t=0}^{\infty}{\gamma}^{t}g(s,t,\mu_{\theta})I(s,s^{'},t,\mu_{\theta})$ converges.
We describe the proof sketch here and the complete proof is referred to Appendix A.
For each state $s'$, which is reached from the initial state $s$ with infinite steps, there are three cases due to deterministic state transitions:
%Due to the deterministic state transitions, there are three cases of each state $s^{'}$ from the initial state $s$ with infinite steps:
never visited, visited once, and visited infinite times.
It is easy to see that $A(s)$ converges in the first two cases.
In the last case, as $A(s)$ is the sum of the power of the matrix ${\gamma}^{t_2}g(s,t_2,\mu_{\theta})$, then we get a upper bound of $\gamma$ such that $A(s)$ converges. By Lebesgue's Dominated Convergence Theorem \cite{royden1968real}, we exchange the order of the limit and the integration,
$
\bigtriangledown_{s}V^{\mu_{\theta}}(s)=\int_{\mathcal{S}}\sum_{t=0}^{\infty}{\gamma}^{t}g(s,t,\mu_{\theta})I(s,s^{'},t,\mu_{\theta})\bigtriangledown_{s^{'}}r(s^{'},\mu_{\theta}(s^{'}))ds^{'}.
$
By the continuity of $T$, $r$ and $\mu_{\theta}$, the gradient of $V^{\mu_{\theta}}(s)$ over $s$ exists.
\end{proof}

Note that the condition proposed in Lemma 1 is indeed necessary in Example \ref{example}, where $n=2, c=2$ and the gradient exists if and only if the discount factor $\gamma<1/4$.
By Lemma \ref{lemma1}, we show that the deterministic policy gradient exists and obtain the closed form. The proof is referred to Appendix B.

\begin{theorem}
\label{theorem_deter}
For any policy $\mu_{\theta}$ and MDP with deterministic state transitions, for any discount factor $\gamma$ in $[0,\frac{1}{nc})$, the policy gradient exists, and

\begin{equation*}
\label{close_j}
\bigtriangledown_{\theta}J(\mu_{\theta})=\int_{\mathcal{S}}\rho^{\mu_{\theta}}(s)\bigtriangledown_{\theta}\mu_{\theta}(s)(\bigtriangledown_{a}r(s,a)|_{a=\mu_{\theta}(s)}+\gamma  \bigtriangledown_{a} T(s,a)|_{a=\mu_{\theta}(s)} \bigtriangledown_{s^{'}}V^{\mu_{\theta}}(s^{'})|_{s^{'}=T(s,a)})ds.
\end{equation*}
\end{theorem}

\section{Deterministic Policy Gradients with general state transitions}
In this section we consider a general setting where the state transition for any state $s$ and any action $a$ is a convex combination of a deterministic transition function $T(s,a)$ with probability $f(s,a)$, and a stochastic probability transition density function $p(s^{'}|s,a)$ with probability $1-f(s,a)$. Note that this setting generalizes that of DPG.
Here, $T$ also satisfies the same condition as in Section \ref{Deter}. We assume that $f(s,a)$, $\bigtriangledown_{s}f(s,a)$ and $\bigtriangledown_{a}f(s,a)$ are continuous and bounded.
By the similar technique to the setting with deterministic state transitions, we get the main theorem which proves the existence of the gradient of $J(\mu_{\theta})$ over $\theta$ for a set of discount factors and proposes two conditions such that for all discount factors the policy gradient exists:
%We define two conditions refered within the text.

{\bf{Condition A.1}}:  $\max_{s}f(s,\mu_{\theta}(s))\leq \frac{1}{nc}$.

{\bf{Condition A.2}}: For any sequence of states $(s_0,...,s_{t-1})$ and any timestep $t$, the eigenvalues of $\prod_{i=0}^{t-1}f(s_i,\mu_{\theta}(s_i))\bigtriangledown_{s_i}T(s_i,\mu_{\theta}(s_i))$ are in $[-1,1]$.

\begin{theorem}{\bf The GDPG Theorem}
\label{theorem_convex}

For any MDP in the general cases and any policy $\mu_{\theta}$, for any discount factor $\gamma$ in $[0,\frac{1}{nc\max_{s}f(s,\mu_{\theta}(s))})$, the policy gradient exists. If the MDP satisfies {{Condition A.1}} or {{Condition A.2}}, for any discount factor and any policy $\mu_{\theta}$, the policy gradient exists. The form is \begin{equation}
\begin{split}
\label{G_J}
\bigtriangledown_{\theta}J(\mu_{\theta})=&\int_{\mathcal{S}}\rho^{\mu_{\theta}}(s)(\bigtriangledown_{\theta}\mu_{\theta}(s)\bigtriangledown_{a}r(s,a)|_{a=\mu_{\theta}(s)}+\gamma f(s,\mu_{\theta}(s)) \bigtriangledown_{\theta}\mu_{\theta}(s) 
\bigtriangledown_{a} T(s,a)|_{a=\mu_{\theta}(s)}\\& \bigtriangledown_{s^{'}}V^{\mu_{\theta}}(s^{'})|_{s^{'}=T(s,a)}+\gamma (1-f(s,\mu_{\theta}(s)))\int_{\mathcal{S}}^{}\bigtriangledown_{\theta}\mu_{\theta}(s)\bigtriangledown_{a}p(s^{'}|s,a)|_{a=\mu_{\theta}(s)}\\&V^{\mu_{\theta}}(s^{'})ds^{'}+\gamma \bigtriangledown_{\theta}f(s,\mu_{\theta}(s))V^{\mu_{\theta}}(s^{'})|_{s^{'}=T(s,\mu_{\theta}(s))}-\gamma \bigtriangledown_{\theta}f(s,\mu_{\theta}(s))\\&\int_{\mathcal{S}}^{}p(s^{'}|s,a)|_{a=\mu_{\theta}(s)}V^{\mu_{\theta}}(s^{'})ds^{'})ds=\int_{\mathcal{S}}\rho^{\mu_{\theta}}(s)(\bigtriangledown_{\theta}\mu_{\theta}(s)\bigtriangledown_{a}Q^{\mu_{\theta}}(s,a)|_{a=\mu_{\theta}(s)})ds.
\end{split}
\end{equation}

\end{theorem}
The proof is referred to Appendix C. It is interesting to note that the form is the same as the form of gradient of DPG.   In fact, the assumption of the condition A.1 and A.2 would become weaker when the probability of the deterministic state transition becomes lower. In the extreme case, i.e., the stochastic case, where the probability is zero, the policy gradient exists without any assumption as in \cite{silver2014deterministic}. In fact, the form of the policy gradient is the same in settings of the deterministic state transition and the general case. However, given an estimator of the value function, the complexity of calculating the gradient of these two cases is different. By comparing (\ref{close_j}) with (\ref{G_J}), we get that it is the more computationally expensive for the gradient of the general case than the deterministic case. The gradient of deterministic state transitions only involves $\bigtriangledown_{\theta}r(s,\mu_{\theta}(s))$ and $\bigtriangledown_{s^{'}}V^{\mu_{\theta}}(s^{'})$, while the gradient of the general case introduces additional integration on the state space.

\subsection{Direct Model-based Extension of DDPG}
As discussed before, even for the environment with low-dimensional state space, the sample complexity of DDPG is significantly high for the general case,  which may limit the capability of the model-free algorithms due to slow convergence.
%which makes the model-free algorithms converge slowly.
Thus, we consider a model-based augmented MDP $\mathcal{M_*}$ of the original MDP $\mathcal{M}$ with the same reward function, while the state transition function is defined as the expectation of the distribution of the next state of the original MDP, i.e., $T_{*}(s,a)= \mathbb{E}[s^{'}|s,a]$.
$\mathcal{M_*}$ is easier to solve  as the state transition of $\mathcal{M^*}$ is deterministic.
Note that if the environment is indeed deterministic, $\mathcal{M_*}=\mathcal{M}$. Now we define a direct model-based extension of DDPG, called MDPG.
MDPG directly uses the gradient of the long-term rewards of $\mathcal{M_*}$ with policy $\mu_{\theta}$ to improve the policy instead of the deterministic policy gradient, i.e.,
$
 \bigtriangledown_{\theta}J_{*}(\mu_{\theta})= \mu_{\theta}(s)\bigtriangledown_{a}Q_{*}^{\mu_{\theta}}(s,a),
$
where $Q_{*}^{\mu_{\theta}}(s,a)$ denotes the action value function of the augmented MDP. However, it is hard to represent the transition dynamics in complex environments, and it may cause the policy to move to a wrong direction as shown in Section 5.2 on problems with large state space.

\subsection{The GDPG Algorithm}
\begin{algorithm}[!htbp]
\small
  \label{alg:GDPG}
  \caption{GDPG algorithm}
  Initialize a positive weight $\alpha$\\
  Initialize the transition network $T(s,a|{\theta}^{T})$ with random weights ${\theta}^{T}$\\
  Initialize the original and augmented critic networks $Q(s,a|{\theta}_{}^{Q})$, $Q_*(s,a|{\theta}^{Q_*})$ with random weights ${\theta}_{}^{Q}$, ${\theta}^{Q_*}$\\
  Initialize the actor network $\mu(s|{\theta}^{\mu})$ with random weights ${\theta}^{\mu}$\\
  Initialize the target networks  $Q_{}^{'}$, ${Q_*}^{'}$ and $\mu^{'}$ with weights ${\theta}_{}^{Q^{'}}={\theta}^{Q}, {\theta}^{Q_*^{'}}={\theta}^{Q_*},
  {\theta}^{{\mu}^{'}}={\theta}^{\mu}$\\
  Initialize Experience Replay buffer $\mathcal{B}$ \\
  \For{episode$=0, ..., N-1$} {
     Initialize a random process $\mathcal{N}$ for action exploration \\
     Receive initial observation state $s_0$.\\
    \For{$t=1, ..., T$} {
		Select action $a_t=\mu_(s_t|{\theta}^{\mu})+{\mathcal{N}}_t$ according to the current policy and exploration noise\\
		Execute action $a_t$, observe reward $r_t$ and new state $s_{t+1}$, and store transition $(s_t,a_t,r_t,s_{t+1})$ in $\mathcal{B}$\\
		Sample a random minibatch of $N$ transitions $(s_i,a_i,r_i,s_{i+1})$ from $\mathcal{B}$\\
		Set $y_i=r_i+\gamma Q_{}^{'}(s_{i+1},{\mu}^{'}(s_{i+1}|{\theta}^{{\mu}^{'}})|{\theta}_{}^{Q^{'}})$\\
		Update the critic $Q$ by minimizing the loss: $L_1=\frac{1}{N}\sum_{i}^{}{(y_i-Q_{}(s_i,a_i|{\theta}^{Q}))}^{2}$\\
		Set $y_i^{'}=r_i+\gamma  Q_{*}^{'}(T(s_i,a_i|{\theta}^{T}) ,{\mu}^{'}(T(s_i,a_i|{\theta}^{T})|{\theta}^{{\mu}^{'}})|{\theta}^{Q_{*}^{'}})$\\
		Update the augmented critic $Q_*$ by minimizing the loss: $L_2=\frac{1}{N}\sum_{i}^{}{(y_i^{'}-Q_{*}(s_i,a_i|{\theta}^{Q_*}))}^{2}$\\
		Upate the transition $T$ by minimizing the loss: $L_3=\frac{1}{N}\sum_{i}^{}{(s_{i+1}-T(s_i,a_i|{\theta}^{T}))}^{2}$\\
		Update the actor by the sampled policy gradient and target networks:
		\begin{equation*}
		\begin{split}
		\bigtriangledown_{\theta^{\mu}}J(\theta^{\mu})=&\frac{1}{N}\sum_{i}^{}(1-\alpha) \bigtriangledown_{{\theta}^{\mu}}\mu(s|{\theta}^{\mu}) \bigtriangledown_{a}Q_{*}(s,a|{\theta}_{}^{Q_*}) + \alpha \bigtriangledown_{{\theta}^{\mu}}\mu(s|{\theta}^{\mu}) \bigtriangledown_{a}Q_{}(s,a|{\theta}_{}^{Q})
		\end{split}
		\end{equation*}
		
		$${\theta}_{}^{Q^{'}}=\tau {\theta}_{}^{Q}+(1-\tau) {\theta}_{}^{Q^{'}}; {\theta}_{}^{Q_*^{'}}=\tau {\theta}_{}^{Q_*}+(1-\tau) {\theta}_{}^{Q_*^{'}}; {\theta}^{\mu^{'}}=\tau {\theta}^{\mu}+(1-\tau) {\theta}^{\mu^{'}}$$
    }
  }

\end{algorithm}

On the one hand, only solving the model-based augmented MDP may be too myopic.
On the other hand, the model-free algorithm suffers from high sample complexity as mentioned.
Consequently, we consider a program that maximizes the long-term rewards of the augmented MDP, with the constraint being that the long-term rewards of the augmented MDP is less than the original MDP, i.e.,

\begin{equation}
\label{program}
 \max_{\theta} J_{*}(\mu_{\theta}), \ s.t.  J_{*}(\mu_{\theta}) \leq J(\mu_{\theta}).
\end{equation}

It is easy to check that the optimum of this program is less than $\max_{\theta} J(\mu_{\theta})$, and it serves as a lower bound of the long-term rewards of the original MDP. 
The intuition of this program is to optimize a model-based objective which is easier to solve and the improvement of the new objective guarantees the improvement of the original objective. 

If the value function is convex in states \footnote{The value functions of Linear Quadratic Regulation \cite{bradtke1993reinforcement} and Linearly-solvable Markov Decision Process \cite{todorov2007linearly} are indeed convex.}, the long-term rewards of $\mathcal{M_*}$ with policy $\mu_{\theta}$, $J_{*}(\mu_{\theta})$ is no larger than the long-term rewards of $\mathcal{M}$, as illustrated in Theorem \ref{Lower bound}. 
That is, the program turns into a problem that maximizes the model-based objective.  The proof is referred to Appendix D.
\begin{theorem}
\label{Lower bound}
If \ $V^{\mu_{\theta}}(s)$ is convex in $s$, $J(\mu_{\theta})\geq J_{*}(\mu_{\theta}).$
\end{theorem}

In the other case that the value function is not convex, it is hard to solve the program directly. 
Therefore, we choose to optimize its Lagrangian dual program, 

\begin{equation}
\label{dual}
\min_{\alpha\geq 0}\max_{\theta}J_{*}(\mu_{\theta})+\alpha(J(\mu_{\theta})-J_{*}(\mu_{\theta})).
\end{equation}

Then for each choice of $\alpha$, we use the gradient of $J_{*}(\mu_{\theta})+\alpha(J(\mu_{\theta})-J_{*}(\mu_{\theta}))$, i.e.,

\begin{equation}
\label{dual1}
(1-\alpha) \mu_{\theta}(s)\bigtriangledown_{a}Q_{*}^{\mu_{\theta}}(s,a) + \alpha\mu_{\theta}(s)\bigtriangledown_{a}Q^{\mu_{\theta}}(s,a),
\end{equation}

which generalizes the gradient of the DDPG algorithm, to improve the policy by stochastic gradient ascent, where $Q_{*}^{\mu_{\theta}}(s,a)$ denotes the action value function of the augmented MDP.
However, the estimation of the value function of the augmented MDP relies on the expectation of the distribution of the next state, which is unknown.
To overcome this challenge, we follow the idea of \cite{nagabandi2017neural}, where neural networks are applied to predict the next state. Different from \cite{nagabandi2017neural} where they take model predictive control as the control policy, we apply the estimators of state transitions to estimate the action-value function of the augmented MDP. We now propose the Generalized Deterministic Policy Gradient (GDPG) algorithm, as shown in Algorithm \ref{alg:GDPG}. Apart from training the actor and the critic, we also train a transition network $T$ which predicts the next state.

\section{Experiments}
In this section, we design a series of experiments to evaluate GDPG.
We aim to investigate the following questions:
(1) How does the value of $\alpha$ affect the performance on a toy problem with general state transitions? \label{item:alpha}
(2) How does GDPG compare with DDPG, MDPG, and other state-of-the-art methods on continuous control benchmarks? We first illustrate the influence of the weight $\alpha$ in a toy environment, ComplexPoint-v0 with general state transitions. Then we evaluate GDPG in a number of continuous control benchmark tasks in OpenAI Gym \cite{1606.01540}, including a classic control problem \cite{moore1990efficient} and a task in the Box2D and MuJoCo \cite{todorov2012mujoco} simulator. The details of our benchmarks are referred to Appendix E. We compare GDPG with the following baselines: (a) DDPG,  (b) MDPG, (c) TRPO, (d) ACKTR.
For the experiments, we run each algorithm 1M steps on each environment over 5 random seeds. Note that the configuration of GDPG is the same as that of DDPG of except for the transition network. Full configuration is referred to Appendix E.  We use the averaged return of previous 100 episodes as the performance metric.

\subsection{The ablation study of GDPG}

To better understand the effect of $\alpha$ in the dual function, we evaluate GDPG with five different choices of the weight $\alpha=0, 0.25, 0.5, 0.75, 1, 2$ in  ComplexPoint-v0.  Figure \ref{fig: toy_env}(a) shows a snapshot of this environment, where the state is the coordinates of the agent in the 5D space while the feasible action set is $[-0.1, 0.1]^5$. The state transition is a convex combination of the deterministic transition $T(s,a)=s+a$ with probability $f(s,a)$, and uniform distribution $[-1,1]^{5}$ with probability $1-f(s,a)$, where $f(s,a)=||a||_2^{2} / 0.05$. The reward function is $r(s,a)=-||s+a||_2$, i.e., the distance between the agent and the origin. The task is terminated either when the agent enters the termination area or the number of steps exceeds a threshold of 100 steps.  Figure \ref{fig: toy_env}(b) shows the performance comparison, and Figure \ref{fig: toy_env}(c) and Figure \ref{fig: toy_env}(d) correspond to its earlier stage and convergence stage, which illustrates convergence and performance more clearly.
As shown, for $\alpha=1$, which indeed corresponds to DDPG, results in a bad performance and slow convergence.
The slow convergence attributes to the computation complexity of gradient in this environment.
For $\alpha=0$, the goal corresponds to optimize the augmented MDP, which performs better than DDPG as it efficiently reduces sample complexity.
However, it is too myopic as it solely focuses on the augmented MDP, which may deviate from the original objective and limit its performance.
We observe that the best performance is achieved when $\alpha=0.5$. We can view the weighted objective as a convex combination of the model-free objective and the model-based objective when $\alpha \in [0,1]$. $\alpha$ trades-off between the convergence and the performance. A large $\alpha$ may introduce bias while a small $\alpha$ may suffer from sample complexity. Note that the choice of $2$ for the value of $\alpha$ achieves the worst performance. Recall (\ref{dual1}), the reason is that setting a value of $\alpha$ larger than 1 may lead the gradient of the policy to a totally opposite direction and induce large variance of the policy gradient.

\begin{figure}
    \centering
    \subfigure[The ComplexPoint environment.]{
        \begin{minipage}[t]{0.32\linewidth}
            \centering
            \includegraphics[scale=0.3]{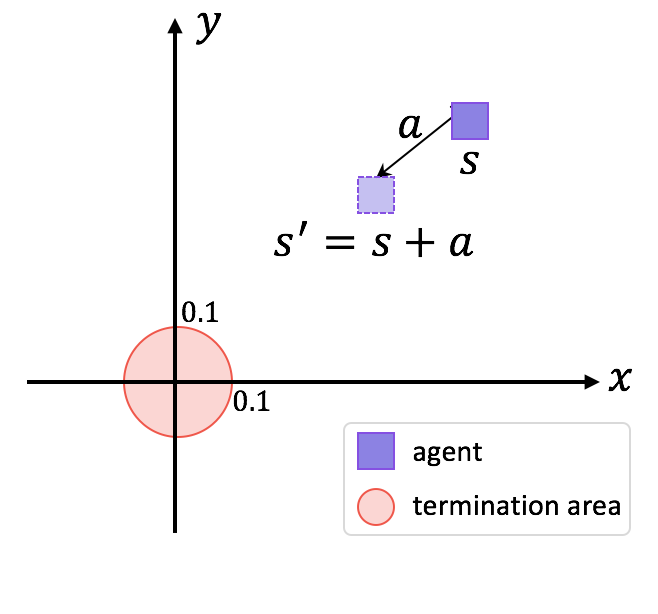}
        \end{minipage}
    }%
    \subfigure[Effect of $\alpha$.]{
        \begin{minipage}[t]{0.32\linewidth}
            \centering
            \includegraphics[scale=0.3]{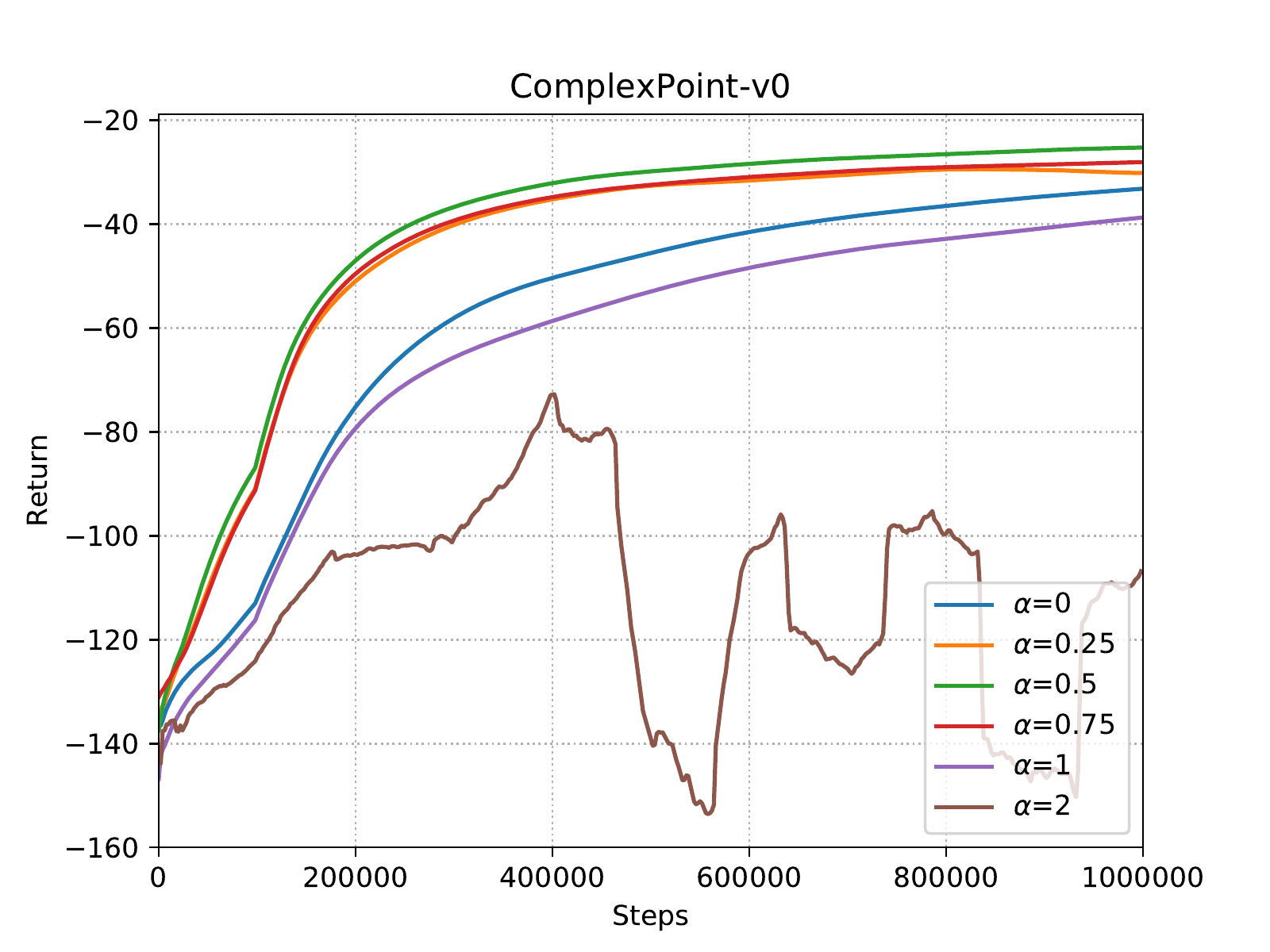}
        \end{minipage}
    }%

    \subfigure[Earlier stage.]{
        \begin{minipage}[t]{0.35\linewidth}
            \centering
            \includegraphics[scale=0.3]{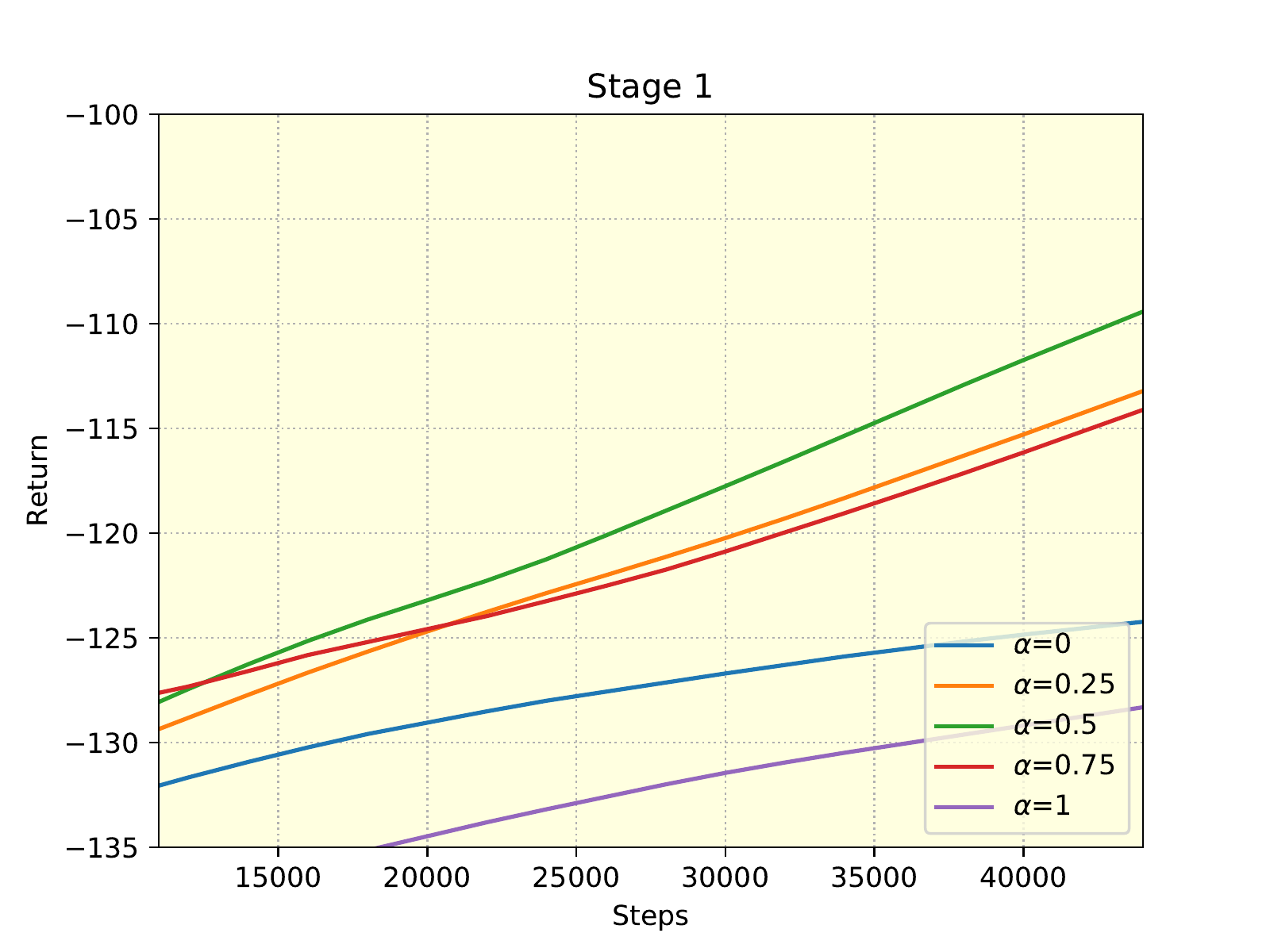}
        \end{minipage}
    }%
    \subfigure[Convergence stage.]{
        \begin{minipage}[t]{0.32\linewidth}
            \centering
            \includegraphics[scale=0.3]{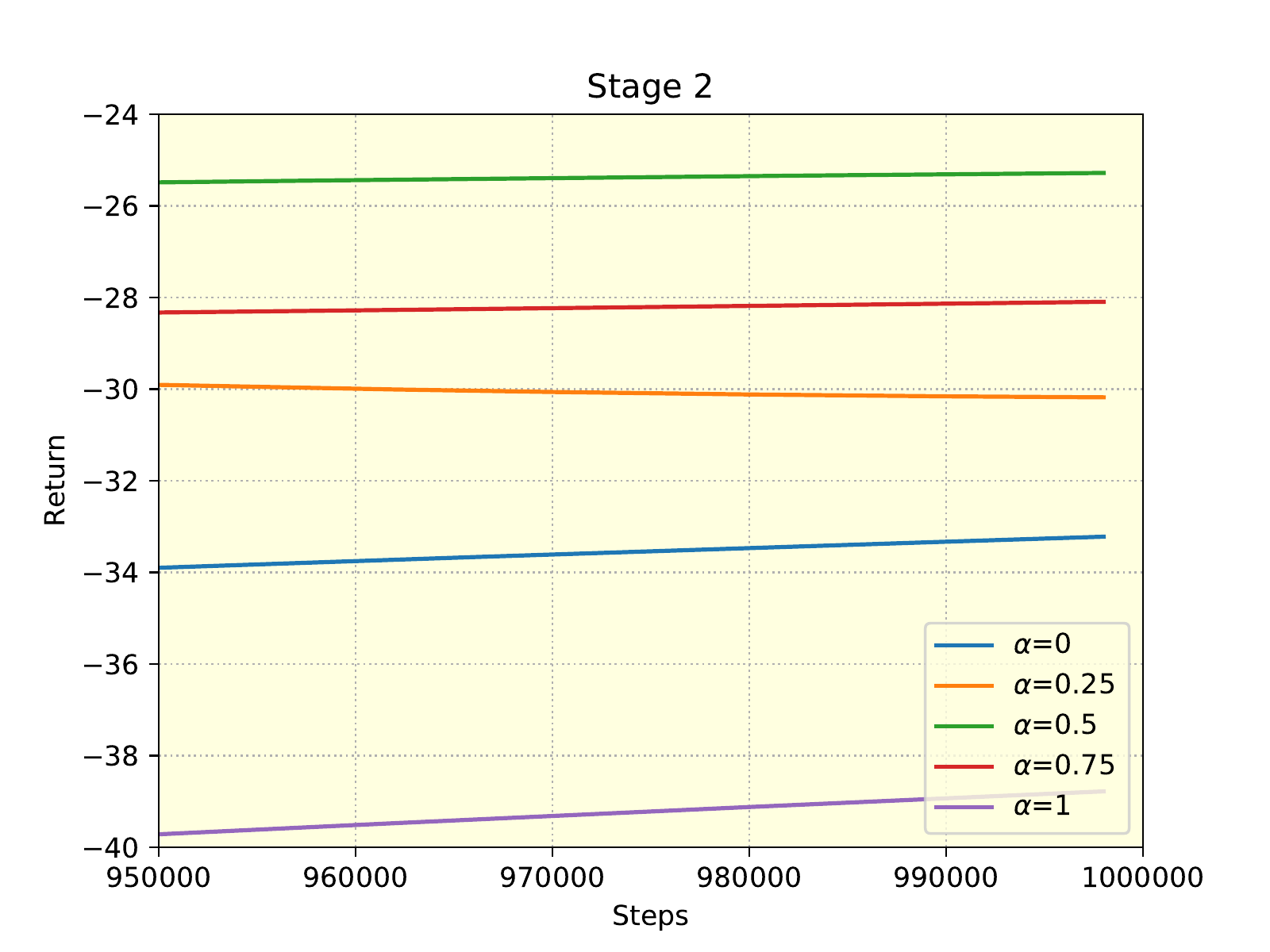}
        \end{minipage}
    }
    \caption{Return/steps of training on algorithms}
    \label{fig: toy_env}
\end{figure}

\subsection{Performance comparison with baselines on continuous control benchmarks}
We now present and discuss the findings from our experiments on several continuous control tasks, all of which are standard benchmark defined in OpenAI Gym \cite{1606.01540}.
Tasks range from low-dimensional input space to high-dimensional input space.
For the baselines algorithms, we use the implementation from OpenAI Baselines \cite{baselines}.
Figure \ref{fig: mujoco2}, \ref{fig: mujoco3} and \ref{fig: mujoco4} show the sample mean and the standard deviation of the averaged returns in each environment.
As shown in Figure \ref{fig: mujoco2}, GDPG outperforms other baselines in tasks with low-dimensional input space including a classic continuous control task and a task simulated by Box2D.
From Figure \ref{fig: mujoco3} and \ref{fig: mujoco4}, we observe that GDPG outperforms high-dimensional tasks simulated by MuJoCo by a large margin, especially in Swimmer-v2, HalfCheetah-v2, and Humanoid-v2. This demonstrates that GDPG combines the model-based augmented MDP and the original MDP efficiently. 
Note that the direct model-based extension of DDPG, MDPG performs the worst in all environments except Swimmer-v2.
 It shows that the model-based technique can not solve complex settings like MuJoCo as it is hard to represent the transition dynamics.

\begin{figure}
    \centering
     \subfigure[Pendulum-v0.]{
        \begin{minipage}[t]{0.45\linewidth}
            \centering
            \includegraphics[scale=0.35]{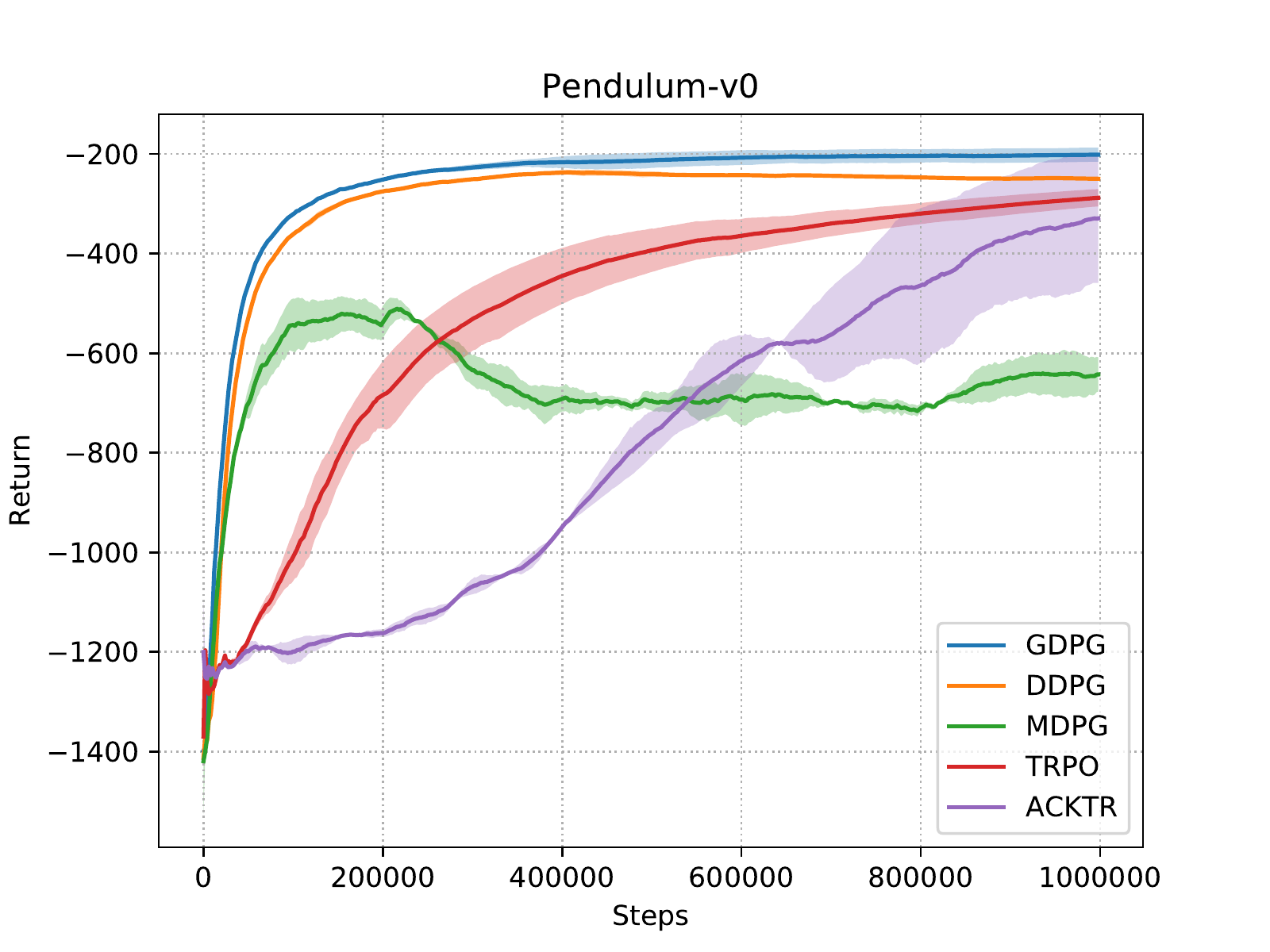}
        \end{minipage}
    }%
    \subfigure[LunarLander-v2.]{
        \begin{minipage}[t]{0.45\linewidth}
            \centering
            \includegraphics[scale=0.35]{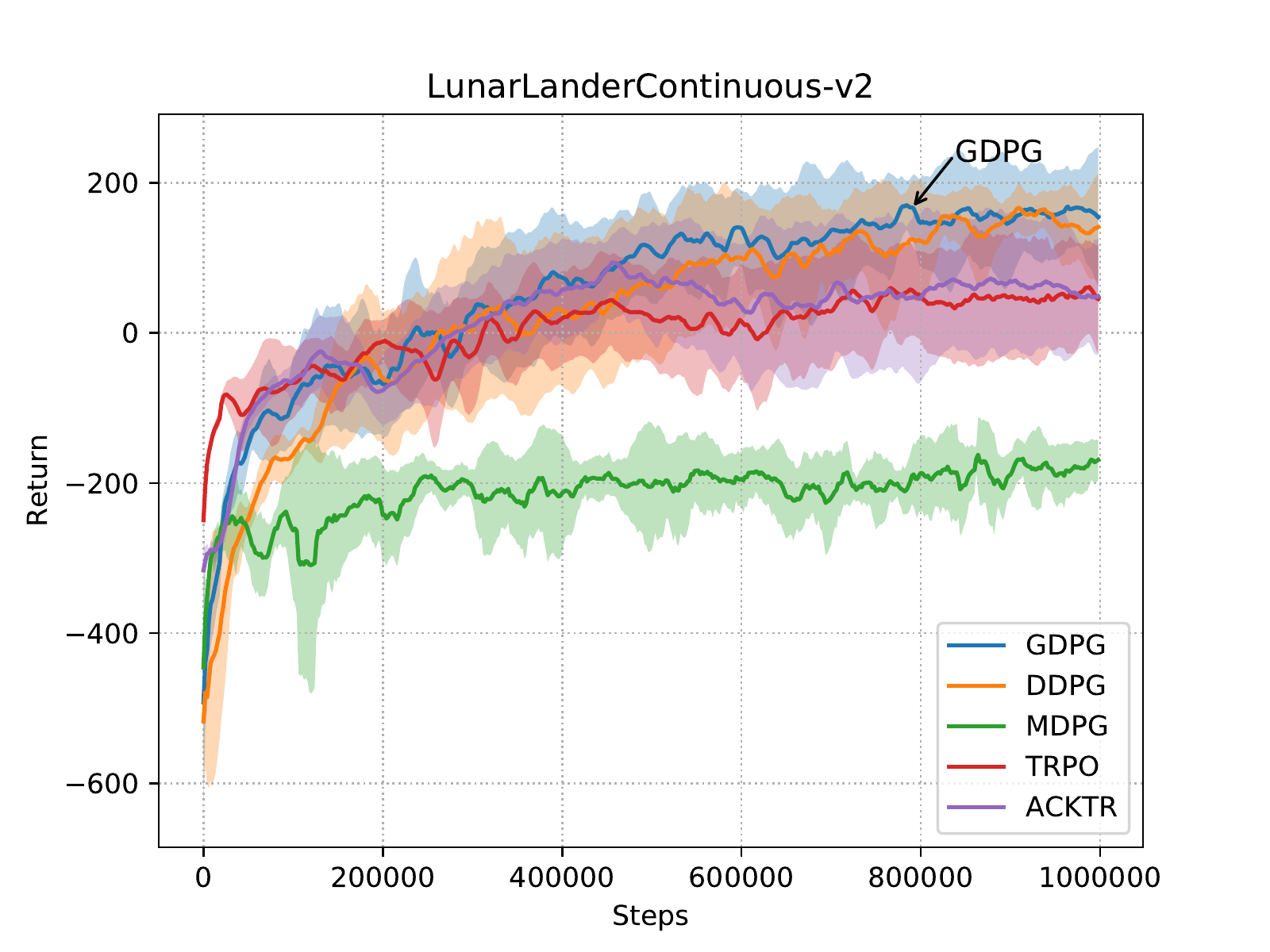}
        \end{minipage}
    }%
 
    \caption{Return/steps of training on environments from the MuJoCo simulator.}
    \label{fig: mujoco2}
\end{figure}

\begin{figure}
    \centering
    
    \subfigure[Swimmer-v2.]{
        \begin{minipage}[t]{0.45\linewidth}
            \centering
            \includegraphics[scale=0.35]{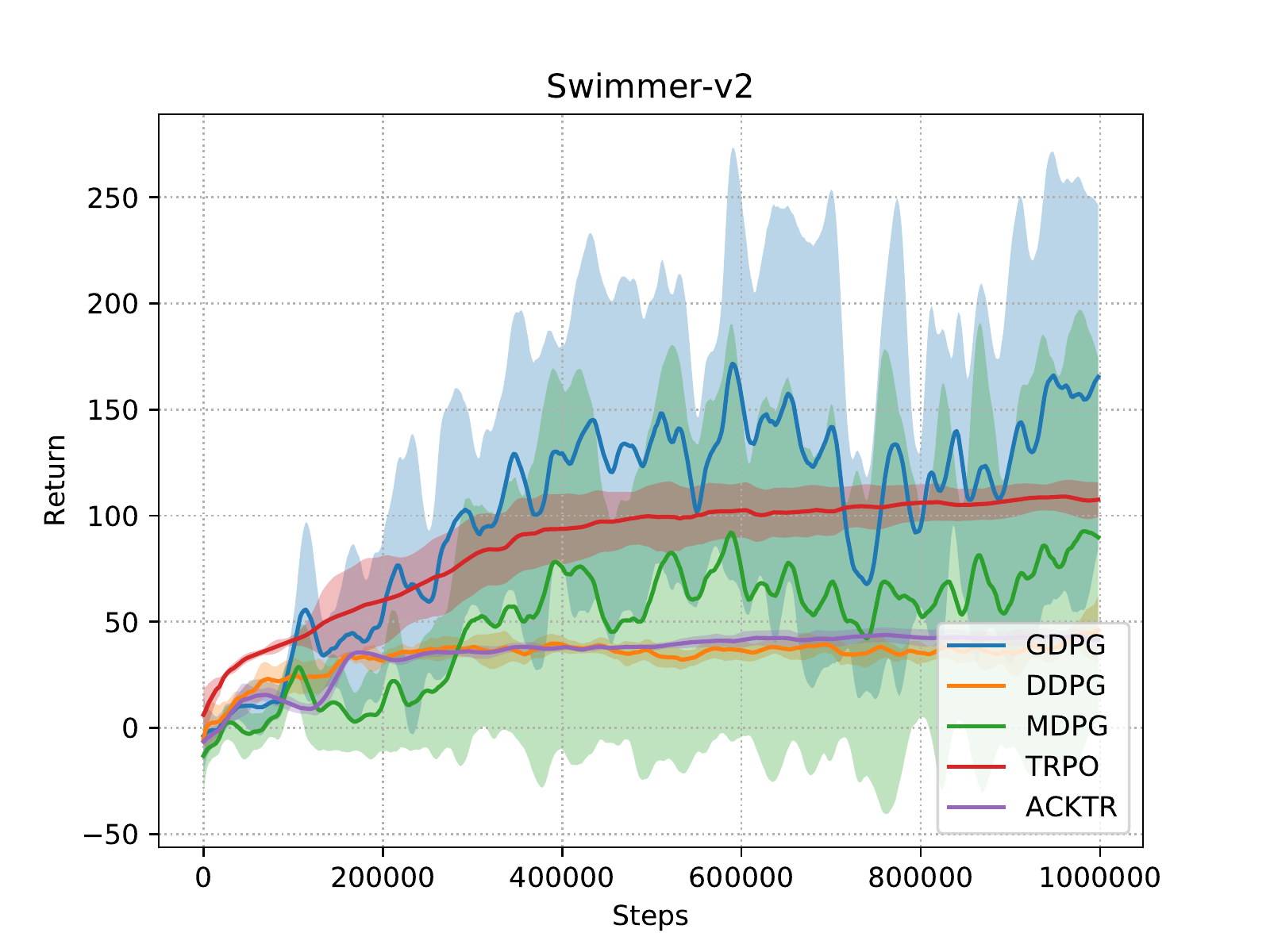}
        \end{minipage}
    }
    \subfigure[HalfCheetah-v2.]{
        \begin{minipage}[t]{0.45\linewidth}
            \centering
            \includegraphics[scale=0.35]{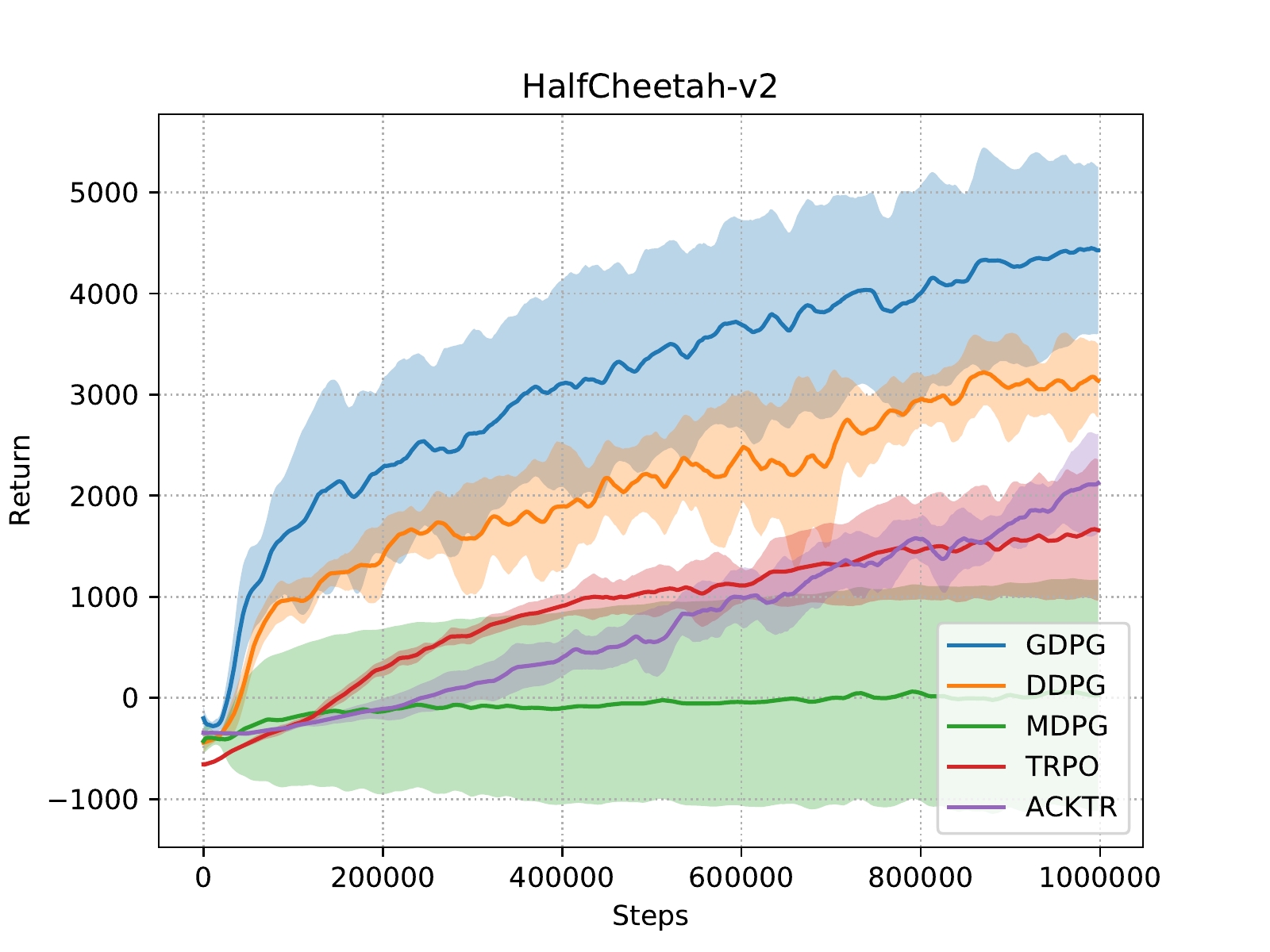}
        \end{minipage}
    }%
    
    \caption{Return/steps of training on environments from the MuJoCo simulator.}
    \label{fig: mujoco3}
\end{figure}

\begin{figure}
    \centering
    
    \subfigure[HumanoidStandup-v2.]{
        \begin{minipage}[t]{0.45\linewidth}
            \centering
            \includegraphics[scale=0.35]{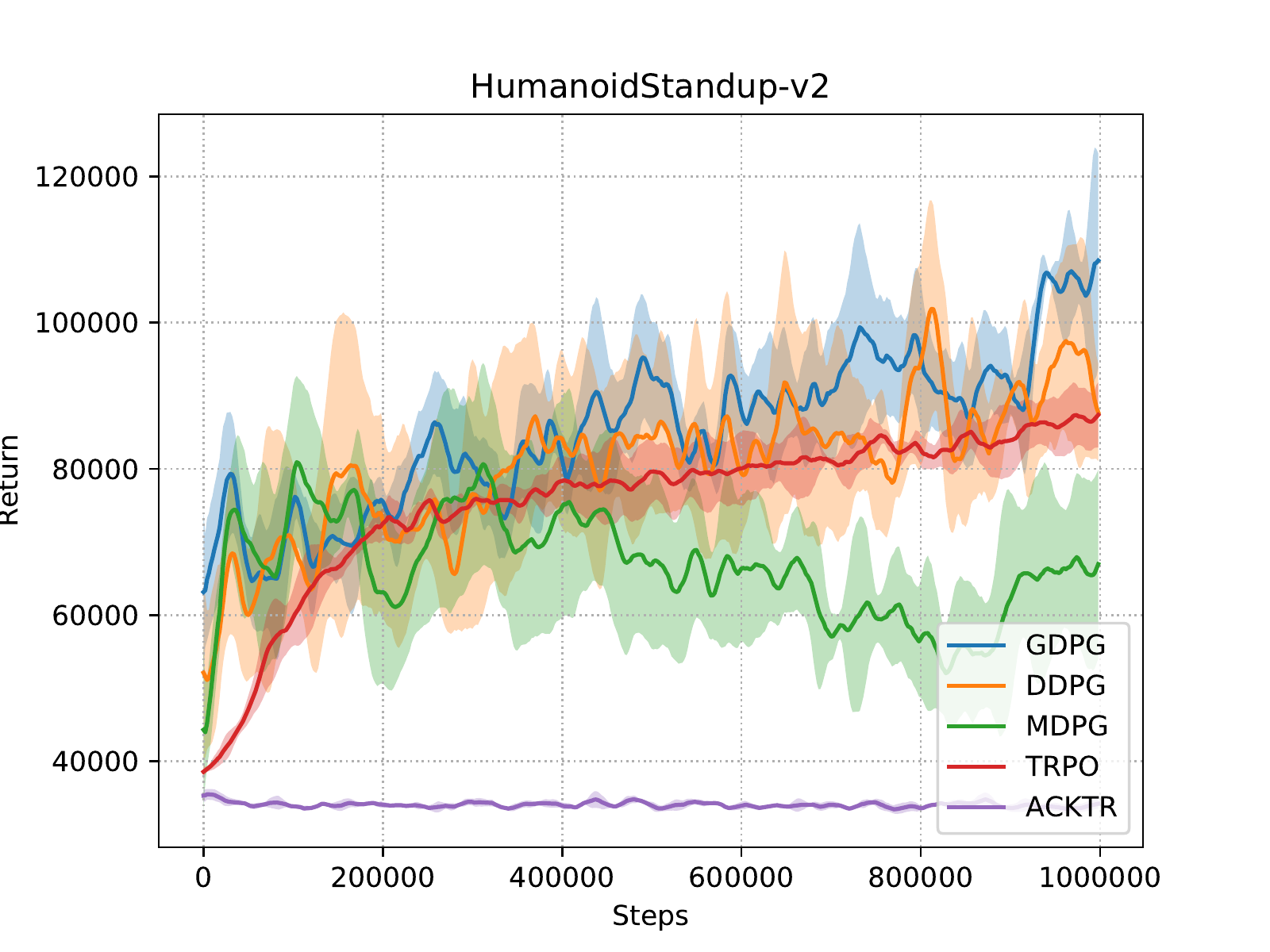}
        \end{minipage}
    }%
    \subfigure[Humanoid-v2.]{
        \begin{minipage}[t]{0.45\linewidth}
            \centering
            \includegraphics[scale=0.35]{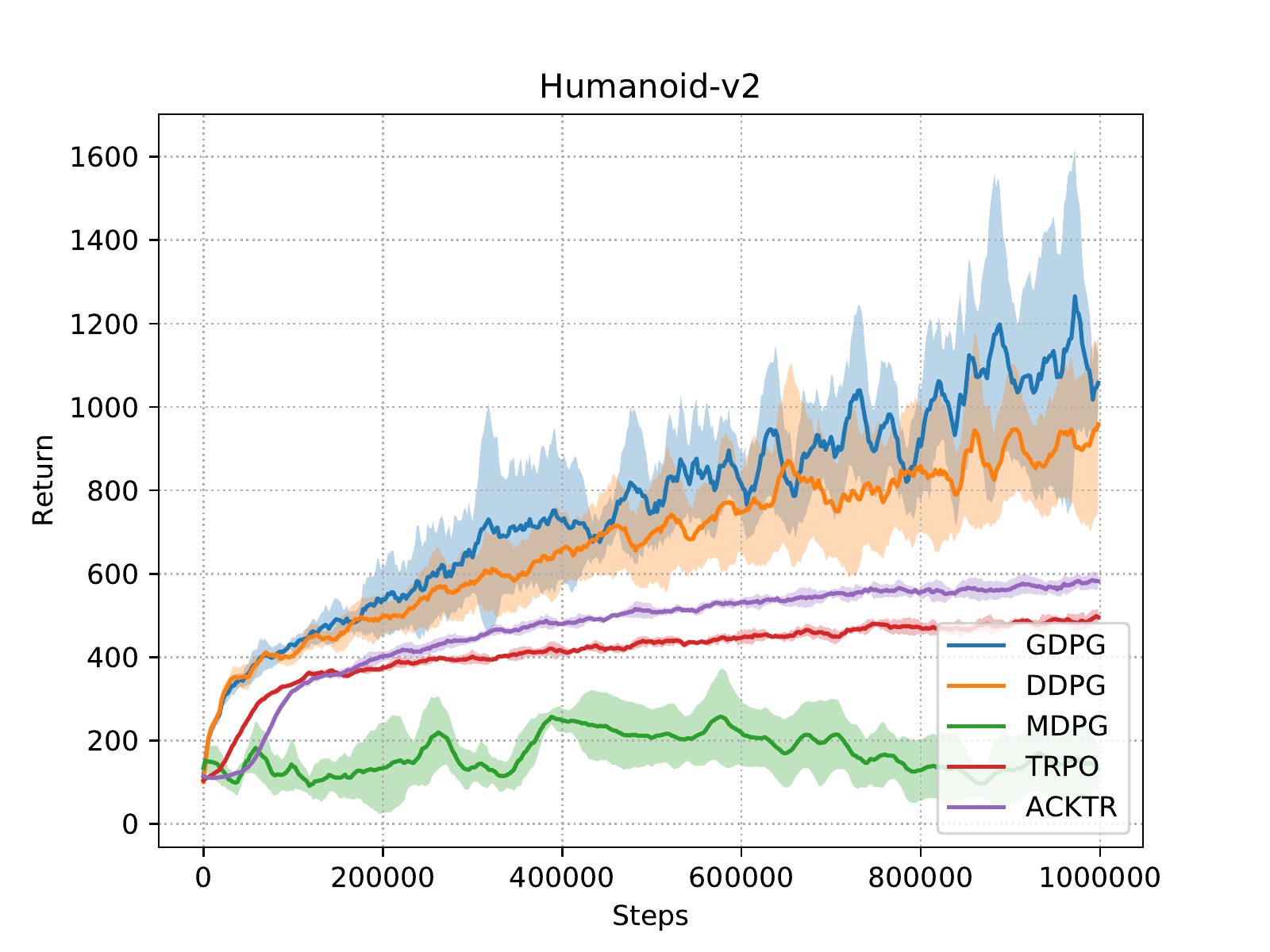}
        \end{minipage}
    }
	    
    \caption{Return/steps of training on environments from the MuJoCo simulator.}
    \label{fig: mujoco4}
\end{figure}

\section{Related Work}

Model-based algorithms has been widely studied \cite{koutnik2013evolving,lioutikov2014sample,moldovan2015optimism,montgomery2016guided} in recent years.
Iterative LQG \cite{li2004iterative} applies model-based methods and assumes a specific form of both transition dynamics and the value function while \cite{sutton1990integrated, gu2016continuous, kurutach2018model} generate synthetic samples by the learned model.
Different from traditional model-based methods, we optimize the dual function that involves the model-based augmented MDP and the original MDP.
Perhaps the most related model-based approach to our work is PILCO \cite{deisenroth2011pilco}, which learns the transition model by Gaussian processes. With the non-parametric transition model, \cite{deisenroth2011pilco} applies policy improvement on analytic policy gradients. However this method does not scale well to nonlinear transition dynamics or high-dimensional state spaces.
Different from \cite{deisenroth2011pilco}, we do not rely on assumptions of the transition model.

\section{Conclusion}
Most existing works on policy gradient assume stochastic state transitions, while most realistic settings often involve deterministic state transitions. 
In this paper, we study a setting with a general state transition that is a convex combination of a stochastic continuous function and a deterministic discontinuous function. 
We prove the existence of the deterministic policy gradient for a certain set of discount factors. 
We propose the GDPG algorithm to reduce the sample complexity of the deterministic policy gradient. 
GDPG solves a program that maximizes the long-terms rewards of the model-based augmented MDP with the constraint that the objective serves as the lower bound of the original MDP.
We compare GDPG with MDPG and state-of-the-art algorithms on several continuous control benchmarks.
Results show that GDPG substantially outperforms other baselines in terms of convergence and long-term rewards. 
For future work, how to address the optimal weight in the dual program remains to be studied.  It is worth studying whether the deterministic policy gradient exists in more general settings that involve multiple deterministic state transitions. Last but not least, it is promising to apply the model-based technique presented in this paper to other model-free algorithms.

\bibliographystyle{abbrv}
\bibliography{gdpg}

\appendix
\section{Proof of Lemma 1}

\begin{proof}
Recall the definition of $V^{\mu_{\theta}}(s)$, we have
\begin{equation}
V^{\mu_{\theta}}(s)=Q^{\mu_{\theta}}(s,\mu_{\theta}(s))=r(s,\mu_{\theta}(s))+\gamma V^{\mu_{\theta}}(s^{'})|_{s^{'}=T(s,\mu_{\theta}(s))}).
\end{equation}
\begin{equation}
\label{v_s1}
\begin{split}
\bigtriangledown_{s}V^{\mu_{\theta}}(s)&=\bigtriangledown_{s}r(s,\mu_{\theta}(s))+\gamma \bigtriangledown_{s}T(s,\mu_{\theta}(s)) \bigtriangledown_{s^{'}}V^{\mu_{\theta}}(s^{'})|_{s^{'}=T(s,\mu_{\theta}(s))}.
%&=\bigtriangledown_{s}r(s,a)|_{a=\mu_{\theta}(s)}+\bigtriangledown_{a}r(s,a)|_{a=\mu_{\theta}(s)}\bigtriangledown_{s}\mu_{\theta}(s)+\gamma c  \bigtriangledown_{s^{'}}V^{\mu_{\theta}}(s^{'})|_{s^{'}=T(s,\mu_{\theta}(s))}.
\end{split}
\end{equation}
%Then if we choose a value of $\gamma$ such that $\gamma c<1$, it is easy to see that
By unrolling (\ref{v_s1}) with infinite steps, we get
\begin{equation}
\bigtriangledown_{s}V^{\mu_{\theta}}(s)=\sum_{t=0}^{\infty}\int_{\mathcal{S}}{\gamma}^{t}g(s,t,\mu_{\theta})I(s,s^{'},t,\mu_{\theta})\bigtriangledown_{s^{'}}r(s^{'},\mu_{\theta}(s^{'}))ds^{'},
\end{equation}
where $I(s,s^{'},t,\mu_{\theta})$ is an indicator function that indicates whether $s^{'}$ is obtained after $t$ steps from the state $s$ following the policy $\mu_{\theta}$.
Here, $g(s,t,\mu_{\theta})=\prod_{i=0}^{t-1}\bigtriangledown_{s_i}T(s_i,\mu_{\theta}(s_i)), $where $s_0=s$ and $s_i$ is the state after $i$ steps following policy $\mu_{\theta}$.
The state transitions and policies are both deterministic. 
We now prove that for any $\mu_{\theta},s,s^{'}$ and any discount factor $\gamma \in [0,\frac{1}{nc})$ such that 
$\sum_{t=0}^{\infty}{\gamma}^{t}g(s,t,\mu_{\theta})I(s,s^{'},t,\mu_{\theta})$ converges. 

For each state $s'$, which is reached from the initial state $s$ with infinite steps, there are three cases due to deterministic state transitions, as analyzed below:

%Due to deterministic state transitions, there are three cases of each state $s^{'}$ from the initial state $s$ with infinite steps: 
\begin{enumerate}
\item Never visited: $\sum_{t=0}^{\infty}{\gamma}^{t}g(s,t,\mu_{\theta})I(s,s^{'},t,\mu_{\theta})=\mathbf{0}.$
\item Visited once: Let $t_{s'}$ denote the number of steps that it takes to reach the state $s'$, then 
$\sum_{t=0}^{\infty}{\gamma}^{t}g(s,t,\mu_{\theta})I(s,s^{'},t,\mu_{\theta})={\gamma}^{t_{s^{'}}}g(s,t_{s^{'}},\mu_{\theta}).$
\item Visited infinite times: Let $t_1$ denote the number of steps it takes to reach $s'$ for the first time. 
The state $s'$ will be revisited every $t_2$ steps after the previous visit. 

%Let $t_{1}$ denote the number of passed steps as the state firstly reaches $s^{'}$, and after that the state will reach $s^{'}$ after each certain number $t_2$of steps. Then %Let $t_2$ denote the number of steps. Then
\begin{equation}
\sum_{t=0}^{\infty}{\gamma}^{t}g(s,t,\mu_{\theta})I(s,s^{'},t,\mu_{\theta})=\sum_{k=0}^{\infty}{\gamma}^{t_1+kt_2}g(s,t_1+kt_2,\mu_{\theta}).
\end{equation}
By the definition of $g$, $g(s,t_1+kt_2,\mu_{\theta})=g(s,t_1,\mu_{\theta}){(g(s,t_2,\mu_{\theta}))}^{k}$, we have
\begin{equation}
\sum_{t=0}^{\infty}{\gamma}^{t}g(s,t,\mu_{\theta})I(s,s^{'},t,\mu_{\theta})={\gamma}^{t_1}g(s,t_1,\mu_{\theta})\sum_{k=0}^{\infty}{({\gamma}^{t_2}g(s,t_2,\mu_{\theta}))}^{k}.
\end{equation}

We get the sum of the absolute value of a row or a column of the matrix $g(s,t_2,\mu_{\theta}))$ is no larger than
%less than or equal to 
${(nc)}^{t_2}$. If we choose $\gamma$ such that $\gamma<\frac{1}{nc}$, by \cite{farnell1944limits}, the absolute value of 
any eigenvalue of ${\gamma}^{t_{2}}g(s,t_{2},\mu_{\theta})$ is strictly less than ${\gamma}^{t_{2}}{(nc)}^{t_2}=1$.

By representing ${\gamma}^{t_{2}}g(s,t_{2},\mu_{\theta})$ with Jordan normal form, i.e., ${\gamma}^{t_{2}}g(s,t_{2},\mu_{\theta})=MJM^{-1}$,

\begin{equation}
{\gamma}^{t_1}g(s,t_1,\mu_{\theta})\sum_{k=0}^{\infty}{({\gamma}^{t_2}g(s,t_2,\mu_{\theta}))}^{k}={\gamma}^{t_1}g(s,t_1,\mu_{\theta})M\sum_{k=0}^{\infty}J^{k} M^{-1}.
\end{equation}

As the absolute value of any eigenvalue of ${\gamma}^{t_{2}}g(s,t_{2},\mu_{\theta})$ is strictly less than $1$, $\sum_{k=0}^{\infty}J^{k}$ converges, then $\sum_{k=0}^{\infty}{({\gamma}^{t_2}g(s,t_2,\mu_{\theta}))}^{k}$  and $\sum_{t=0}^{\infty}{\gamma}^{t}g(s,t,\mu_{\theta})I(s,s^{'},t,\mu_{\theta})$ converge. 
\end{enumerate}

By the Lebesgue's Dominated Convergence Theorem \cite{royden1968real}, we exchange the order of the limit and the integration.
\begin{equation}
\bigtriangledown_{s}V^{\mu_{\theta}}(s)=\int_{\mathcal{S}}\sum_{t=0}^{\infty}{\gamma}^{t}g(s,t,\mu_{\theta})I(s,s^{'},t,\mu_{\theta})\bigtriangledown_{s^{'}}r(s^{'},\mu_{\theta}(s^{'}))ds^{'}.
\end{equation}
By the continuity of $T$ , $r$ and $\mu_{\theta}$, the gradient of $V^{\mu_{\theta}}(s)$ over $s$ exists.

\end{proof}

\section{Proof of Theorem 1}

\begin{proof}
By the definition,
\begin{equation}
\label{V1}
\begin{split}
\bigtriangledown_{\theta}V^{\mu_{\theta}}(s)=&\bigtriangledown_{\theta}Q^{\mu_{\theta}}(s,\mu_{\theta}(s))\\
=&\bigtriangledown_{\theta}(r(s,\mu_{\theta}(s))+\gamma V^{\mu_{\theta}}(s^{'})|_{s^{'}=T(s,\mu_{\theta}(s))})\\
=&\bigtriangledown_{\theta}\mu_{\theta}(s)\bigtriangledown_{a}r(s,a)|_{a=\mu_{\theta}(s)}+\gamma \bigtriangledown_{\theta}V^{\mu_{\theta}}(s^{'})|_{s^{'}=T(s,\mu_{\theta}(s))}\\
+&\gamma \bigtriangledown_{\theta}\mu_{\theta}(s) \bigtriangledown_{a}T(s,a)|_{a=\mu_{\theta}(s)} \bigtriangledown_{s^{'}}V^{\mu_{\theta}}(s^{'})|_{s^{'}=T(s,a)}.
\end{split}
\end{equation}

With the indicator function $I(s,s^{'},t,\mu_{\theta})$, we rewrite the equation (\ref{V1}):
\begin{equation}
\label{gra_v1}
\begin{split}
\bigtriangledown_{\theta}V^{\mu_{\theta}}(s)=&\bigtriangledown_{\theta}\mu_{\theta}(s)(\bigtriangledown_{a}r(s,a)|_{a=\mu_{\theta}(s)}+\gamma \bigtriangledown_{a} T(s,a)|_{a=\mu_{\theta}(s)} \bigtriangledown_{s^{'}}V^{\mu_{\theta}}(s^{'})|_{s^{'}=T(s,a)})\\
&+ \int_{\mathcal{S}}^{}\gamma I(s,s^{'},1,\mu_{\theta})\bigtriangledown_{\theta}V^{\mu_{\theta}}(s^{'})ds^{'}.\\
=&\bigtriangledown_{\theta}\mu_{\theta}(s)(\bigtriangledown_{a}r(s,a)|_{a=\mu_{\theta}(s)}+\gamma \bigtriangledown_{a} T(s,a)|_{a=\mu_{\theta}(s)} \bigtriangledown_{s^{'}}V^{\mu_{\theta}}(s^{'})|_{s^{'}=T(s,a)})\\
&+  \int_{\mathcal{S}}^{}\gamma I(s,s^{'},1,\mu_{\theta})\bigtriangledown_{\theta}\mu_{\theta}(s^{'})(\bigtriangledown_{a^{'}}r(s^{'},a^{'})|_{a^{'}=\mu_{\theta}(s^{'})}+\gamma \bigtriangledown_{a^{'}} T(s^{'},a^{'})|_{a^{'}=\mu_{\theta}(s^{'})}\\
&\bigtriangledown_{s^{''}}V^{\mu_{\theta}}(s^{''})|_{s^{''}=T(s^{'},a^{'})}
)ds^{'}+  \int_{\mathcal{S}}^{}\gamma I(s,s^{'},1,\mu_{\theta}) \int_{\mathcal{S}}^{}\gamma I(s^{'},s^{''},1,\mu_{\theta})\bigtriangledown_{\theta}V^{\mu_{\theta}}(s^{''})ds^{''}ds^{'}.\\
=&\bigtriangledown_{\theta}\mu_{\theta}(s)(\bigtriangledown_{a}r(s,a)|_{a=\mu_{\theta}(s)}+\gamma \bigtriangledown_{a} T(s,a)|_{a=\mu_{\theta}(s)} \bigtriangledown_{s^{'}}V^{\mu_{\theta}}(s^{'})|_{s^{'}=T(s,a)})\\
&+  \int_{\mathcal{S}}^{}\gamma I(s,s^{'},1,\mu_{\theta})\bigtriangledown_{\theta}\mu_{\theta}(s^{'})(\bigtriangledown_{a^{'}}r(s^{'},a^{'})|_{a^{'}=\mu_{\theta}(s^{'})}+\gamma \bigtriangledown_{a^{'}} T(s^{'},a^{'})|_{a^{'}=\mu_{\theta}(s^{'})}\\ &\bigtriangledown_{s^{''}}V^{\mu_{\theta}}(s^{''})|_{s^{''}=T(s^{'},a^{'})}
)ds^{'}
+  \int_{\mathcal{S}}^{}\gamma^{2} I(s,s^{''},2,\mu_{\theta})\bigtriangledown_{\theta}V^{\mu_{\theta}}(s^{''})ds^{''}.
\end{split}
\end{equation}

By unrolling (\ref{gra_v1}) with infinite steps, we get

\begin{equation}
\label{close_v1}
\begin{split}
\bigtriangledown_{\theta}V^{\mu_{\theta}}(s)&=\int_{\mathcal{S}}\sum_{t=0}^{\infty}{\gamma}^{t}I(s,s^{'},t,\mu_{\theta})\bigtriangledown_{\theta}\mu_{\theta}(s^{'})(\bigtriangledown_{a^{'}}r(s^{'},a^{'})|_{a^{'}=\mu_{\theta}(s^{'})}+\gamma
\bigtriangledown_{a^{'}} T(s^{'},a^{'})|_{a^{'}=\mu_{\theta}(s^{'})}\\ & \bigtriangledown_{s^{''}}V^{\mu_{\theta}}(s^{''})|_{s^{''}=T(s^{'},a^{'})}
)ds^{'}.
\end{split}
\end{equation}

By the definition of $J(\mu_{\theta})$, \begin{equation}
\begin{split}
\bigtriangledown_{\theta}J(\mu_{\theta})=&\bigtriangledown_{\theta}\int_{\mathcal{S}}^{}p_0(s)V^{\mu_{\theta}}(s)ds\\
=&\int_{\mathcal{S}}^{}p_0(s)\bigtriangledown_{\theta}V^{\mu_{\theta}}(s)ds.
\end{split}
\end{equation}
As \begin{equation}
\rho^{\mu_{\theta}}(s^{'})=\int_{\mathcal{S}}^{}\sum_{t=0}^{\infty}{\gamma}^{t}p_0(s)I(s,s^{'},t,\mu_{\theta})ds.
\end{equation}
By exchanging the order of the integration, we get 
%(\ref{close_j}).

\begin{equation}
\label{close_jjjj}
\bigtriangledown_{\theta}J(\mu_{\theta})=\int_{\mathcal{S}}\rho^{\mu_{\theta}}(s)\bigtriangledown_{\theta}\mu_{\theta}(s)(\bigtriangledown_{a}r(s,a)|_{a=\mu_{\theta}(s)}+\gamma  \bigtriangledown_{a} T(s,a)|_{a=\mu_{\theta}(s)} \bigtriangledown_{s^{'}}V^{\mu_{\theta}}(s^{'})|_{s^{'}=T(s,a)})ds.
\end{equation}

\end{proof}

\section{Proof of Theorem 2}

\begin{proof}
We first prove a fact that for any continuous policy $\mu_{\theta}$, there exists a discount fator $\gamma$ such that the gradient of $V^{\mu_{\theta}}(s)$ over $s$ exists.
Recall the definition of $V^{\mu_{\theta}}(s)$, we have

\begin{equation}
\begin{split}
V^{\mu_{\theta}}(s)&=Q^{\mu_{\theta}}(s,\mu_{\theta}(s))\\&=r(s,\mu_{\theta}(s))+\gamma f(s,\mu_{\theta}(s)) V^{\mu_{\theta}}(s^{'})|_{s^{'}=T(s,\mu_{\theta}(s))}+ \gamma (1-f(s,\mu_{\theta}(s)))\\&\int_{\mathcal{S}}^{}p(s^{'}|s,a)|_{a=\mu_{\theta}(s)}V^{\mu_{\theta}}(s^{'})ds^{'}.
\end{split}
\end{equation}

Then
\begin{equation}
\label{v_s2}
\begin{split}
\bigtriangledown_{s}V^{\mu_{\theta}}(s)=&\bigtriangledown_{s}r(s,\mu_{\theta}(s))+ \gamma \bigtriangledown_{s}f(s,\mu_{\theta}(s))V^{\mu_{\theta}}(s^{'})|_{s^{'}=T(s,\mu_{\theta}(s))} + \gamma f(s,\mu_{\theta}(s))\\& \bigtriangledown_{s}T(s,\mu_{\theta}(s)) \bigtriangledown_{s^{'}}V^{\mu_{\theta}}(s^{'})|_{s^{'}=T(s,\mu_{\theta}(s))} + \gamma (1-f(s,\mu_{\theta}(s)))\int_{\mathcal{S}}^{}\bigtriangledown_{s}p(s^{'}|s,\mu_{\theta}(s))\\&V^{\mu_{\theta}}(s^{'})ds^{'}-\gamma \bigtriangledown_{s}f(s,\mu_{\theta}(s))\int_{\mathcal{S}}^{}p(s^{'}|s,a)|_{a=\mu_{\theta}(s)}V^{\mu_{\theta}}(s^{'})ds^{'}.
\end{split}
\end{equation}

%Then if we choose a value of $\gamma$ such that $\gamma c<1$, it is easy to see that
By unrolling (\ref{v_s2}) with infinite steps, we get

\begin{equation}
\begin{split}
\bigtriangledown_{s}V^{\mu_{\theta}}(s)=&\sum_{t=0}^{\infty}\int_{\mathcal{S}}{\gamma}^{t}g(s,t,\mu_{\theta})I(s,s^{'},t,\mu_{\theta})(\bigtriangledown_{s^{'}}r(s^{'},\mu_{\theta}(s^{'}))+ \gamma \bigtriangledown_{s^{'}}f(s^{'},\mu_{\theta}(s^{'}))V^{\mu_{\theta}}(s^{''})+ \\& \gamma (1-f(s^{'},\mu_{\theta}(s^{'})) \int_{\mathcal{S}}^{} \bigtriangledown_{s^{'}} p(s^{''}\mid  {s}^{'},\mu_{\theta}(s^{'}))V^{\mu_{\theta}}(s^{''})ds^{''}-\gamma \bigtriangledown_{s}f(s^{'},\mu_{\theta}(s^{'}))\\&\int_{\mathcal{S}}^{}p(s^{''}|s^{'},a^{'})|_{a^{'}=\mu_{\theta}(s^{'})}V^{\mu_{\theta}}(s^{''} )ds^{''})ds,
\end{split}
\end{equation}

where $I(s,s^{'},t,\mu_{\theta})$ is an indicator function that indicates whether $s^{'}$ is obtained after $t$ steps from the state $s$ following the policy $\mu_{\theta}$ and the deterministic transition. 
$g(s,t,\mu_{\theta})=\prod_{i=0}^{t-1}f(s_i,\mu_{\theta}(s_i))\bigtriangledown_{s_i}T(s_i,\mu_{\theta}(s_i)),$ where $s_0=s$.
Here, as the policy is deterministic and the calculation of the gradient with $\theta$ only involves the deterministic state transitions, $s_i$ is the state after $i$ steps following policy $\mu_{\theta}$. 
By the same technique of the proof of Lemma 1, we get that there exists a discount factor $\gamma(0<\gamma<1)$ such that 
\begin{equation}
\sum_{t=0}^{\infty}{\gamma}^{t}g(s,t,\mu_{\theta})I(s,s^{'},t,\mu_{\theta})
\end{equation}
converges. In fact, we can choose $\gamma$ such that $\gamma \times \max_{s}f(s,\mu_{\theta}(s)) < \frac{1}{nc}$, where $n$ denotes the dimension of the state, and $c$ be the maximum absolute value of elements
of all matrices $ \bigtriangledown_{s}T(s,\mu_{\theta}(s))$.

If the condition A.1 holds, i.e., for any state $s$, $\max_{s} f(s,\mu_{\theta}(s))  \leq \frac{1}{nc}$, by the proof of Lemma 1, for any discount factor, $\sum_{t=0}^{\infty}{\gamma}^{t}g(s,t,\mu_{\theta})I(s,s^{'},t,\mu_{\theta})$ converges.

If the condition A.2 holds, we have $$\gamma^{t_2}\text{max}_{\lambda}|\lambda(g(s,t_2,\mu_{\theta}))|<1.$$  Thus for any discount factor $
\sum_{t=0}^{\infty}{\gamma}^{t}g(s,t,\mu_{\theta})I(s,s^{'},t,\mu_{\theta})$ converges.

By the Lebesgue's Dominated Convergence Theorem, we exchange the order of the limit and the intergation:
\begin{equation}
\begin{split}
\bigtriangledown_{s}V^{\mu_{\theta}}(s)=&\int_{\mathcal{S}}\sum_{t=0}^{\infty}{\gamma}^{t}g(s,t,\mu_{\theta})I(s,s^{'},t,\mu_{\theta})(\bigtriangledown_{s^{'}}r(s^{'},\mu_{\theta}(s^{'}))+ \gamma \bigtriangledown_{s^{'}}f(s^{'},\mu_{\theta}(s^{'}))V^{\mu_{\theta}}(s^{''})+ \\& \gamma (1-f(s^{'},\mu_{\theta}(s^{'})) \int_{\mathcal{S}}^{} \bigtriangledown_{s^{'}} p(s^{''}\mid  {s}^{'},\mu_{\theta}(s^{'}))V^{\mu_{\theta}}(s^{''})ds^{''}-\gamma \bigtriangledown_{s}f(s^{'},\mu_{\theta}(s^{'}))\\&\int_{\mathcal{S}}^{}p(s^{''}|s^{'},a^{'})|_{a^{'}=\mu_{\theta}(s^{'})}V^{\mu_{\theta}}(s^{''} )ds^{''})ds,
\end{split}
\end{equation}
By the continuity of $T$ , $r$, $f$ and $\mu_{\theta}$, the gradient of $V^{\mu_{\theta}}(s)$ over $s$ exists. 

Now we derive the form of the policy gradient. By definition, \begin{equation}
\label{v_a1}
\begin{split}
\bigtriangledown_{\theta}V^{\mu_{\theta}}(s)=&\bigtriangledown_{\theta}r(s,\mu_{\theta}(s))+ \gamma \bigtriangledown_{\theta}f(s,\mu_{\theta}(s))V^{\mu_{\theta}}(s^{'})|_{s^{'}=T(s,\mu_{\theta}(s))} + \gamma f(s,\mu_{\theta}(s))\\& \bigtriangledown_{\theta}T(s,\mu_{\theta}(s)) \bigtriangledown_{s^{'}}V^{\mu_{\theta}}(s^{'})|_{s^{'}=T(s,\mu_{\theta}(s))} + \gamma (1-f(s,\mu_{\theta}(s)))\int_{\mathcal{S}}^{}\bigtriangledown_{\theta}p(s^{'}|s,\mu_{\theta}(s))\\&V^{\mu_{\theta}}(s^{'})ds^{'}-\gamma \bigtriangledown_{\theta}f(s,\mu_{\theta}(s))\int_{\mathcal{S}}^{}p(s^{'}|s,a)|_{a=\mu_{\theta}(s)}V^{\mu_{\theta}}(s^{'})ds^{'}+\\& \gamma f(s,\mu_{\theta}(s)) \bigtriangledown_{\theta} V^{\mu_{\theta}}(s^{'})|_{s^{'}=T(s,\mu_{\theta}(s))}+ \gamma (1-f(s,\mu_{\theta}(s)))\\&\int_{\mathcal{S}}^{}p(s^{'}|s,a)|_{a=\mu_{\theta}(s)}\bigtriangledown_{\theta}V^{\mu_{\theta}}(s^{'})ds^{'}.
\end{split}
\end{equation}
By unrolling (\ref{v_a1}) with infinite steps, we get 

\begin{equation}
\begin{split}
\bigtriangledown_{\theta}V^{\mu_{\theta}}(s)=&\int_{\mathcal{S}}\sum_{t=0}^{\infty}{\gamma}^{t}p(s,s^{'},t,\mu_{\theta})(\bigtriangledown_{\theta}r(s^{'},\mu_{\theta}(s^{'}))+ \gamma \bigtriangledown_{\theta}f(s^{'},\mu_{\theta}(s^{'}))V^{\mu_{\theta}}(s^{''})|_{s^{''}=T(s,\mu_{\theta}(s^{'}))} \\&+  \gamma f(s^{'},\mu_{\theta}(s^{'})) \bigtriangledown_{\theta}T(s^{'},\mu_{\theta}(s^{'})) \bigtriangledown_{s^{''}}V^{\mu_{\theta}}(s^{''})|_{s^{''}=T(s,\mu_{\theta}(s^{'}))} + \gamma (1-f(s^{'},\mu_{\theta}(s^{'})))\\&\int_{\mathcal{S}}^{}\bigtriangledown_{\theta}p(s^{''}|s^{'},\mu_{\theta}(s^{'}))V^{\mu_{\theta}}(s^{''})ds^{''}-\gamma \bigtriangledown_{\theta}f(s^{'},\mu_{\theta}(s^{'}))\\&\int_{\mathcal{S}}^{}p(s^{''}|s^{'},a)|_{a=\mu_{\theta}(s^{'})}V^{\mu_{\theta}}(s^{''})ds^{''})ds^{'},
\end{split}
\end{equation}

where $p(s,s^{'},t,\mu_{\theta})$ denotes the probability density of the state $s^{'}$ after $t$ steps following the policy $\mu_{\theta}$. By the definition of $J(\mu_{\theta})$ and the same technique as the proof of Theorem 1, we get (\ref{G_J}). By definition, 

\begin{equation}
Q^{\mu_{\theta}}(s,a)=r(s,a)+\gamma f(s,a) V^{\mu_{\theta}}(s^{'})|_{s^{'}=T(s,a)}+ \gamma (1-f(s,a))\int_{\mathcal{S}}^{}p(s^{'}|s,a)V^{\mu_{\theta}}(s^{'})ds^{'}.
\end{equation}

Then 
\begin{equation}
\begin{split}
\bigtriangledown_{\theta}\mu_{\theta}(s)\bigtriangledown_{a}Q^{\mu_{\theta}}(s,a)|_{a=\mu_{\theta}(s)}=&\bigtriangledown_{\theta}\mu_{\theta}(s)\bigtriangledown_{a}r(s,a)|_{a=\mu_{\theta}(s)}+\gamma f(s,\mu_{\theta}(s)) \bigtriangledown_{\theta}\mu_{\theta}(s) 
\\&\bigtriangledown_{a} T(s,a)|_{a=\mu_{\theta}(s)} \bigtriangledown_{s^{'}}V^{\mu_{\theta}}(s^{'})|_{s^{'}=T(s,a)}+\gamma (1-f(s,\mu_{\theta}(s)))\\&\int_{\mathcal{S}}^{}\bigtriangledown_{\theta}\mu_{\theta}(s)\bigtriangledown_{a}p(s^{'}|s,a)|_{a=\mu_{\theta}(s)}V^{\mu_{\theta}}(s^{'})ds^{'}+\gamma \bigtriangledown_{\theta}f(s,\mu_{\theta}(s))\\&V^{\mu_{\theta}}(s^{'})|_{s^{'}=T(s,\mu_{\theta}(s))}-\gamma \bigtriangledown_{\theta}f(s,\mu_{\theta}(s))\int_{\mathcal{S}}^{}p(s^{'}|s,a)V^{\mu_{\theta}}(s^{'})ds^{'}.
\end{split}
\end{equation}

Thus, we get that the policy gradient of (\ref{G_J}) is equivalent to the form of the DPG theorem.
\end{proof}

\section{Proof of Theorem 3}

\begin{proof}
By definition, we have
\begin{equation}
\forall s, V^{\mu_{\theta}}(s)=r(s,\mu_{\theta}(s))+ \gamma \int_{s\sim D(s,\mu_{\theta}(s))}^{} V^{\mu_{\theta}}(s^{'})ds^{'},
\end{equation}
where $D(s,\mu_{\theta}(s))$ denotes the distribution of the next state. As the value function is convex, we get
\begin{equation}
\forall s, V^{\mu_{\theta}}(s)\geq r(s,\mu_{\theta}(s))+ \gamma V^{\mu_{\theta}}(s^{'})|_{s^{'}=T^{*}(s,\mu_{\theta}(s))}.
\end{equation}
By definition,
\begin{equation}
\forall s, V_{*}^{\mu_{\theta}}(s)= r(s,\mu_{\theta}(s))+ \gamma V_{*}^{\mu_{\theta}}(s^{'})|_{s^{'}=T^{*}(s,\mu_{\theta}(s))}.
\end{equation}
Thus 
\begin{equation}
\label{neq1}
\forall s, V^{\mu_{\theta}}(s)-V_{*}^{\mu_{\theta}}(s) \geq \gamma (V^{\mu_{\theta}}(s^{'})-V_{*}^{\mu_{\theta}}(s^{'}))|_{s^{'}=T^{*}(s,\mu_{\theta}(s))}.
\end{equation}

As these two value functions are bounded, there is a lower bound $C$ such that
 \begin{equation}
 \label{bound1}
\forall s, V^{\mu_{\theta}}(s)-V_{*}^{\mu_{\theta}}(s) \geq C.
\end{equation}

Combining (\ref{neq1}) with (\ref{bound1}) repeatedly, we obtain 
\begin{equation}
\forall s, V^{\mu_{\theta}}(s)\geq V_{*}^{\mu_{\theta}}(s).
\end{equation}

Note that
\begin{equation}
J(\mu_{\theta})=\int_{\mathcal{S}}^{}p_0(s)V^{\mu_{\theta}}(s)ds.
\end{equation}
and 
\begin{equation}
J_{*}(\mu_{\theta})=\int_{\mathcal{S}}^{}p_0(s)V_{*}^{\mu_{\theta}}(s)ds.
\end{equation}

Thus $J(\mu_{\theta})\geq J_{*}(\mu_{\theta}).$
\end{proof}

\section{Implementation Details}
In this section we describle the details of the implementation of GDPG.  The configuration of the actor network and the augmented critic network is the same as the implementation of OpenAI Baslines. 
Each network has two fully connected layers, where each layer has 64 units. 
The activation function is RelU, the batch size is $128$, the learning rate of the actor is ${10}^{-4}$, and the learning rate of the critic is ${10}^{-3}$. 

We exploit the model-based technique by estimating the state transition function using deep neural networks. 
For problems with low-diemensional input space including ComplexPoint-v0, Pendulum-v0, HalfCheetah-v2, LunarLanderContinuous-v2, we use the two layers fully connected structure for the transition network. 
For problems which are more complex, including  Humanoid-v2, HumanoidStandup-v2, we apply the Convolutional Neural Networks (CNN). 
In particular, the network contains two layers of CNN followed by a fully connected layer. The configuration for the CNN layer is as listed in Table \ref{table: config}. The learning rate of the transition network is ${10}^{-3}$. We also add $L_2$ norm regularizer to the loss and the batch size is $128$.

Note that the weight of our objective affects the performance of GDPG as discussed in Section 5.3, we test different value of $\alpha$ on all environments, and
we get that the value of $\alpha=0.9$ achieves the best performance in all environments.

\begin{table}
\centering
\begin{tabular}{c|c} 
\hline \hline
\addstackgap[2pt]{\textbf{Paramter}} & \textbf{Value} \\ \hline
\addstackgap[2pt]{Filters for Layer 1} & 32 \\
\addstackgap[2pt]{Filters for Layer 2} & 64 \\
\addstackgap[2pt]{Kernel Size} & 5 \\
\addstackgap[2pt]{Paxdding Mode} & Same \\
\addstackgap[2pt]{Pooling Size} & 2 \\
\addstackgap[2pt]{Strides} & 2 \\
\addstackgap[2pt]{Activation Function} & ReLU\\
\hline \hline
\end{tabular}
\caption{Configurations.}
\label{table: config}
\end{table}

\begin{table}
\centering
\begin{tabular}{l|c|c}
\hline \hline
{\textbf{Environment}} & $\boldsymbol{||\mathcal{S}||}$ & $\boldsymbol{||\mathcal{A}||}$ \\
\hline
%{Point-v0} & 5 & 5 \\
{ComplexPoint-v0} &  5&  5\\ \hline
{Pendulum-v0} & 3 & 1 \\ \hline
{LunarLanderContinuous-v2} & 8 & 2\\ \hline
{Swimmer-v2} & 8 & 2 \\
{HalfCheetah-v2} & 17 & 6 \\
{HumanoidStandup-v2} & 376 & 17 \\
{Humanoid-v2} & 376 & 17 \\
\hline \hline
\end{tabular}
\caption{List of environments.}
\label{table: envs}
\end{table}

\end{document}